\documentclass[10pt,journal,compsoc]{IEEEtran}

\ifCLASSOPTIONcompsoc
  \usepackage[nocompress]{cite}
\else
  \usepackage{cite}
\fi

\ifCLASSOPTIONcompsoc
  \usepackage[caption=false,font=footnotesize,labelfont=sf,textfont=sf]{subfig}
\else
  \usepackage[caption=false,font=footnotesize]{subfig}
\fi

\usepackage{fixltx2e}
\usepackage{pdflscape}

\usepackage{times}
\usepackage{mdframed}
\usepackage{amsmath}
\usepackage{amsthm}
\usepackage{amssymb}
\usepackage{dsfont}
\usepackage{bbm}

\usepackage{textcomp}
\usepackage{stmaryrd}

\usepackage{graphicx}
\usepackage{combinedgraphics}
\usepackage{multirow}
\usepackage{booktabs}
\usepackage{xcolor}
\usepackage{comment}
\usepackage{mathrsfs}
\usepackage{wrapfig} 
\usepackage{rotating}

\usepackage{tcolorbox}
\usepackage{enumitem}

\usepackage[commentsnumbered,linesnumbered,ruled]{algorithm2e}

\SetCommentSty{mycommfont}

\usepackage{hyperref}
\hypersetup{
    colorlinks=true,
    linkcolor=blue,
    filecolor=magenta,      
    urlcolor=cyan,
}
 
\graphicspath{{graphs/}{graphs_extra/}{graphs2/}}
\usepackage{stackengine}
\usepackage{tikz}

\usepackage{array,etoolbox}
\preto\tabular{\setcounter{magicrownumbers}{0}}
\newcounter{magicrownumbers}

\usepackage[switch,pagewise]{lineno} %

\definecolor{gray}{rgb}{0.2, 0.5, 0.478}

\newtheorem{definition}{Definition}
\newtheorem{proposition}{Proposition}

\newtheorem{lem}{Lemma}

\DeclareMathOperator*{\argmin}{argmin}
\DeclareMathOperator*{\argmax}{argmax}

\begin{document} 

\title{Point-Set Kernel Clustering}

\author{Kai Ming Ting, 
Jonathan R. Wells,
and Ye Zhu 
\thanks{K. M. Ting is affiliated with the National Key Laboratory for Novel Software Technology, Nanjing University, China.
E-mail: {tingkm@nju.edu.cn}}%
\thanks{J. R. Wells and Y. Zhu are affiliated with the School of Information Technology, Deakin University, Geelong, Australia.  
E-mail: {jonathan.wells.research@gmail.com, ye.zhu@ieee.org.}}}

\IEEEtitleabstractindextext{%
\begin{abstract}
Measuring similarity between two objects is the core operation in existing clustering algorithms in grouping similar objects into clusters. This paper introduces a new similarity measure called point-set kernel which computes the similarity between an object and a set of objects. The proposed clustering procedure utilizes this new measure to characterize every cluster grown from a seed object. We show that the new clustering procedure is both effective and efficient that enables it to deal with large scale datasets. In contrast, existing clustering algorithms are either efficient or effective.  In comparison with the state-of-the-art density-peak clustering and scalable kernel k-means clustering, we show that the proposed algorithm is more effective and runs orders of magnitude faster  when applying to datasets of millions of  data points, on a commonly used computing machine.

\end{abstract}

}

\maketitle

\IEEEraisesectionheading{\section{Introduction}\label{sec:introduction}}

\IEEEPARstart{S}{imilarity} between two objects is used as the basis in grouping objects into clusters in existing clustering algorithms. State-of-the-art clustering algorithms are density based \cite{DensityPeak-2014}, and they rely on distance measure (aka a similarity measure) between two objects to compute the density of each data point, representing each object. An early influential density based algorithm DBSCAN \cite{DBSCAN-1996} separates core points from noise points by using a density threshold, where the former has high density and the latter has low density. Then only core points in the neighbourhood of each other are grouped into the same cluster. The key advantage of DBSCAN is that it can detect clusters of arbitrary shapes and sizes. The exact condition under which DBSCAN fails to identify all clusters has been determined recently \cite{YeZhu2016}. In general terms, DBSCAN fails to identify all clusters when the clusters have hugely varying densities.

Density-peak clustering (DP) \cite{DensityPeak-2014}, which is a more recent density-based algorithm, begins by identifying density peaks that are far from each other; and then links data points that are transitively connected to a peak to form a cluster. Because it does not employ a density threshold, DP has avoided the weakness of DBSCAN mentioned above. Though DP is a stronger clustering algorithm than DBSCAN in general \cite{IsolationKernel-AAAI2019}, DP has its own weaknesses. A key weakness is the requirement to find all density peaks in the first step, after the density of each point has been estimated based on a distance/similarity measure. As a result, it has difficulty identifying clusters, where each cluster has uniform density distribution. 

A recent research has shown that a kernelized DBSCAN called MBSCAN \cite{IsolationKernel-AAAI2019}, which employs Isolation Kernel \cite{ting2018IsolationKernel}, overcomes the weakness of DBSCAN and uplifts its clustering performance to the same level as DP. However, it has high computational cost because it employs the same algorithm as DBSCAN and only replaces the distance measure with Isolation Kernel. Like DP, DBSCAN and MBSCAN are unable to deal with large scale datasets because they have quadratic time complexity.

In a nutshell, high computational cost is a longstanding fundamental issue of existing density-based clustering algorithms. We contend that the root cause is due to the use of a similarity between two data points, resulting the computational cost to be at least proportional to the square of the data size, i.e.,
$n^2$, where $n$ is the number of data points in a given dataset. 
This has restricted the existing density-based clustering algorithms to small datasets only, as evidenced on the datasets used in their evaluations \cite{DBSCAN-1996, DensityPeak-2014, YeZhu2016, IsolationKernel-AAAI2019}. 

In the age of big data, these density-based algorithms could not be used, despite their superior clustering capability in comparison with more traditional clustering algorithms such as k-means \cite{k-means-macqueen1967}. Although there are attempts to parallelize these algorithms or approximate the clustering outcomes through sampling, the fundamental limitation remains, i.e., their computational cost being at least proportional to $n^2$. In other words, these attempts are a mitigating approach that enables some large datasets to be executed in reasonable time. But huge datasets remain out of bound for these algorithms running on a machine with a fixed number of CPUs (see Section \ref{sec_scaleup_test} for a scaleup test example.)

\begin{mdframed}[backgroundcolor=black!10]
\noindent
\textbf{Significance}\\
\noindent
This paper introduces the first kernel-based clustering which has runtime proportional to data size and yields clustering outcomes that are superior to those of existing clustering algorithms. The success of the new clustering is due to the use of a proposed point-set kernel which has an exact finite dimensional feature map. This enables the new clustering to (i) characterize clusters of arbitrary shapes, varied densities and sizes in a dataset; and (ii) run orders of magnitude faster than existing state-of-the-art clustering algorithms which have quadratic time cost. It is the only clustering algorithm which can process millions of data points on a commonly used machine, as far as we know.
\end{mdframed}

The rest of the paper is organized as follows. The proposed point-set kernel and its properties are described in the next three sections. Section \ref{sec_psKC} presents the proposed clustering algorithm, followed by a conceptual comparison in the next section. Section~\ref{sec_experiment} provides the empirical evaluation results. A discussion of related issues and conclusions are provided in the last two sections.

\section{Proposed point-set kernel}
Rather than relying on a similarity between two data points, we propose a new similarity which measures how similar a data point $x \in \mathbb{R}^d$ is to a set of data points $G$, as a point-set kernel:
\begin{equation}
 \widehat{K}(x, G) 
 =  \left< {\Phi}(x), \widehat{\Phi}(G) \right>
 \label{eqn_dot-prod-mean-map0}
\end{equation}
and
\begin{equation}
\widehat{\Phi}(G)= \frac{1}{|G|}\sum\limits_{y \in G} {\Phi}(y)
 \label{eqn_kernel-mean-map}
\end{equation}
\noindent
where $\widehat{\Phi}$ is the {\em kernel mean map}\footnote{Kernel mean embedding \cite{HilbertSpaceEmbedding2007,KernelMeanEmbedding2017} is an approach to convert a point-to-point kernel into a distribution kernel which measures similarity between two distributions. The proposed point-set kernel can be viewed as a special case of kernel mean embedding. Kernel mean embedding uses the same kernel mean map we have stated here.} of $\widehat{K}$; $\Phi$ is the feature map of a point-to-point kernel $\kappa$; and $\left< a, b \right>$ denotes a dot product between two vectors $a$ and $b$.

In contrast, the point-to-point kernel (a similarity between two data points), expressed as a dot product, is given as follows \cite{SVMBook}:
\begin{equation}
 \kappa(x, y) 
 =  \left< {\Phi}(x), {\Phi}(y) \right>
 \label{eqn_dot-prod-mean-map1}
\end{equation}

Notice that the summation in $\widehat{\Phi}(G)$ (in Equation \ref{eqn_kernel-mean-map}) is to be done once only as a preprocessing. Then computing $\widehat{K}(x, G)$ in equation \ref{eqn_dot-prod-mean-map0}, based on the dot product, takes a fixed amount of time only, independent of $n$ (the data size of $G$.) 

Therefore, to compute the similarity of $x$ with respect to $G$ for all points $x$ in $G$, i,e., $\widehat{K}(x, G)\ \forall x \in G$, has a computational cost which is proportional to $n$ only. 

Also note that the use of the feature map $\Phi$ is necessary in order to achieve the stated efficiency. The alternative, which employs the point-to-point kernel/distance directly in the computation, will have a computational cost that is proportional to $n^2$---the root cause of high computational cost in existing density-based algorithms.

The point-set kernel formulation assumes that the point-to-point kernel $\kappa$ has a finite-dimensional feature map $\Phi$.

Commonly used point-to-point kernels (such as Gaussian and Laplacian kernels) have two key limitations \cite{ting2018IsolationKernel,IsolationKernel-AAAI2019}: they have a feature map of intractable dimensionality \cite{Nystrom_NIPS2000}; and their similarity is independent of a given dataset. The first limitation prevents these kernels to be used in the proposed formulation directly. 

\vspace{3mm}
\noindent
\textbf{$\widehat{K}$ built from Isolation Kernel}

We propose to use a recently introduced point-to-point kernel which has an exact finite dimensional feature map called Isolation Kernel \cite{ting2018IsolationKernel, IsolationKernel-AAAI2019} as $\kappa$ in $\widehat{K}$. It has two characteristics, which are antitheses to the two limitations mentioned above, i.e., it has a finite-dimensional feature map; and its similarity adapts to local density of the data distribution of a given dataset. 

The first characteristic enables Isolation Kernel to be used directly in the proposed point-set kernel.
The exact finite-dimensional feature map is crucial in achieving the only kernel-based clustering which has runtime proportional to data size, we will propose in Section \ref{sec_psKC}.

The second characteristic has a specific data dependent property: \emph{two points in a sparse region are more similar than two points of equal inter-point distance in a dense region} \cite{ting2018IsolationKernel, IsolationKernel-AAAI2019}. This characteristic is crucial for the proposed clustering algorithm to obtain good clustering outcomes.

As the point-set kernel is constructed from a dataset $D$, Equations \ref{eqn_dot-prod-mean-map0} and \ref{eqn_kernel-mean-map} are more precisely expressed as:
\[
\widehat{K}(x,G|D) = \left< {\Phi}(x|D), \widehat{\Phi}(G|D) \right>
\]
and
\[
\widehat{\Phi}(G|D)= \frac{1}{|G|}\sum\limits_{y \in G} {\Phi}(y|D)
\]
\noindent
where $G \subseteq D$; and $\Phi$ is the feature map of Isolation Kernel which is constructed from $D$ (and Isolation Kernel has no functional form.)

The symbol `$|D$' is dropped from the expressions in the rest of the paper for brevity.

The details of Isolation Kernel and its feature map can be found in Appendix \ref{App_IK}.

Once $\widehat{K}$ is derived from a dataset $D$, it is fixed and ready to be used in a kernel-based algorithm.

\vspace{3mm}
\noindent
\textbf{$\widehat{K}$ similarity distribution}

The point-set kernel can be used to describe the data distribution of a dataset in terms of similarity distribution, independent of the clustering process. To do this, some sets in the dataset must be given, either as the ground truth or the clustering outcome of an algorithm. 

\begin{definition}
Given sets $G^j, j=1,\dots,k$ in a dataset $D$ and the $\widehat{K}$ derived from $D$, the  $\widehat{K}$ similarity distribution  is defined as: 
\[
\max_j \widehat{K}(x, G^j), \forall x \in \mathbb{R}^d.
\]

\label{def:similarity_def}
\end{definition}

The distribution is analogous to the density distribution estimated by, e.g.,  a kernel density estimator (KDE); except that sets in a dataset must be provided. In other words, the  $\widehat{K}$ similarity distribution describes the data distribution in terms of the given sets in the dataset.

An example comparison of density distribution and $\widehat{K}$ similarity distribution of a same dataset is given in Figure \ref{fig:example_varied_densities}. 

\begin{figure}
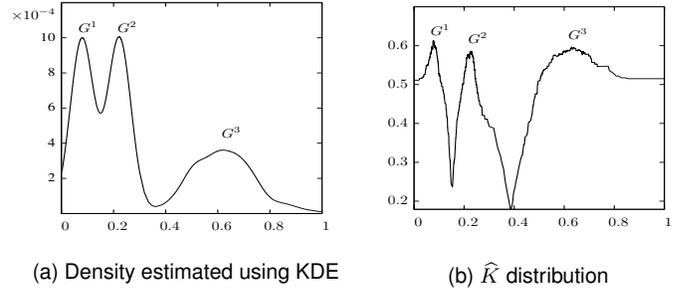

\vspace{-6mm}
 \centering
 \subfloat[Density estimated  using KDE
  ]{\includecombinedgraphics[height=.20\textwidth,width=.25\textwidth,vecfile=hard_3_cls_KDE-eps-converted-to]{hard_3_cls_KDE}} 
 \subfloat[$\widehat{K}$ distribution
 ]{\includecombinedgraphics[height=.20\textwidth,width=.25\textwidth,vecfile=hard_3_cls_1D_sim-eps-converted-to]{hard_3_cls_1D_sim}} \\
 \caption{Density distribution versus $\widehat{K}$ similarity distribution on a one-dimensional dataset having two dense clusters and one sparse cluster.}
 \label{fig:example_varied_densities}
 \vspace{-3mm}
\end{figure}

\section{Properties of point-set kernel for points outside a cluster}
\label{sec_properties}
To use the point-set kernel to grow a cluster, we need to understand the properties of the kernel for points outside the cluster.

Given a dataset $D$ and  an (expanding) cluster $C \subset D$.

Let $x \in D \setminus C$ and $x'\in D \setminus C'$; the distance between $x$  and  a set $C$ be $\ell(x,C) \equiv \ell(x,\Bar{x}_C)$, where $\Bar{x}_C$ is the preimage of $\widehat{\Phi}(C|D)$; and $\rho(C) > \rho(C')$, where $\rho(C)$ denotes the average density of $C$.

\begin{proposition}
The point-set kernel $\widehat{K}(x,C|D)$ derived from $D$ has the following properties:

 \begin{itemize}
    \item[(a)]Fall-off-the-cliff property:  $\widehat{K}(x,C)$ decreases sharply as $\ell(x,C)$ increases.

  \item[(b)] Data dependent property: $\frac{d \widehat{K}(x,C)}{dx} > \frac{d \widehat{K}(x',C')}{dx'}$, if $\ell(x,C) = \ell(x',C')$  and $\rho(C) > \rho(C')$.
\end{itemize}
\end{proposition}

In other words, the rate of falling-off at $x$ is data dependent: it is proportional to the average density of $C$. This property follows directly from the data dependent property of Isolation Kernel. See the proof in Appendix \ref{App_Property_psKC}.

These properties can also be understood from the implementation of Isolation Kernel. For each isolation partitioning, partitions at the boundary of $C$ cover a limited area outside $C$. This contributes to the fall-off-the-cliff property. Also, the dense part of $C$ has small partitions and the sparse part of $C$ has large partitions. This creates different rates of falling-off mentioned above.

An example of the fall-off-the-cliff property is shown in Figure~\ref{fig:Fall-off-the-cliff-property}. It shows that the rate of falling-off is higher wrt the dense cluster than that wrt the sparse cluster.

An example of $\widehat{K}$ distribution, as stated in Definition \ref{def:similarity_def}, is shown in Figure \ref{fig:example_varied_densities}, in comparison with the density distribution as estimated by a kernel density estimator using Gaussian kernel. Notice that the `valley' between the two dense clusters is significantly sharper than that between the sparse cluster and the dense cluster. This is a direct result of the two properties mentioned above.

\begin{figure}
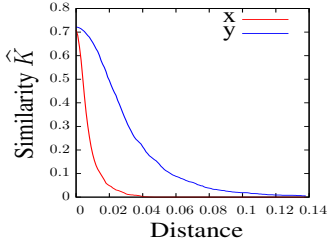

\vspace{-3mm}
 \centering
{\includecombinedgraphics[height=.20\textwidth,width=.25\textwidth,vecfile=hard_3_cls_fig_1_50-eps-converted-to]{hard_3_cls_fig_1_50}}
\vspace{-4mm}
 \caption{Fall-off-the-cliff  property of $\widehat{K}$ : $x$ is close to the dense cluster $C^2$ and $y$ is close to the sparse cluster $C^3$. 
 The x-axis: distances $\ell(x,C^2)$ and $\ell(y,C^3)$ correspond to $x$ and $y$ curves, respectively. $C^2$ and  $C^3$ are subsets of $G^2$ and  $G^3$ (target clusters), respectively, shown in Figure~\ref{fig:example_varied_densities}(a). Each $C$ has the 50 highest density points in  $G$.}
 \label{fig:Fall-off-the-cliff-property}
\end{figure}

The point-set kernel that employs Isolation Kernel has two important influences. First, with an appropriate clustering algorithm design, the above-mentioned properties enable 
\begin{itemize}
  \item Each cluster to be expanded radially in all directions from a seed in multiple iterations, where each iteration recruits a subset of new members in the immediate neighborhood of the expanding cluster. 
  \item Arbitrary-shaped clusters of different  densities and sizes to be discovered via the above cluster expansion.
\end{itemize}

Second, in terms of runtime efficiency, the use of Isolation Kernel in the point-set kernel enables similarity between a point and a set to be computed efficiently, replacing point-to-point similarity/distance as the core operation. The latter is the root cause of high time complexity in existing algorithms.

The proposed point-set kernel clustering algorithm is designed based on these two properties. It is described in the next section.

\section{Clustering based on point-set kernel}
\label{sec_psKC}
The proposed clustering, called \textbf{point-set kernel clustering} or  \texttt{psKC}, employs the point-set kernel $\widehat{K}$ to characterize clusters. It identifies all members of each cluster by first locating the seed of the dataset. Then, it expands its members in the cluster's local neighbourhood 
which grows at a set rate ($\varrho$) 
incrementally; and it stops growing when all unassigned points have similarity wrt the cluster falling below a threshold ($\tau$). 
The process repeats for the next cluster using the remaining points in the given dataset $D$, yet to be assigned to any clusters found so far, until $D$ is empty or no point can be found which has similarity more than $\tau$. All remaining points after the clustering process are noise as they are less than the set threshold for each of the clusters discovered.

The  \texttt{psKC} procedure is shown in Algorithm \ref{alg:pset-KC}.

\begin{algorithm}[t]
\SetAlgoLined
\SetKwInOut{Input}{Input}\SetKwInOut{Output}{Output}
 \Input{$D$: dataset,  $\tau$: similarity threshold, $\varrho$: growth rate}
\Output{$G^j, j=1,\dots,k$: $k$ clusters, $N$: noise set}
\BlankLine
$k=0$\;
\While{$|D|>1$}{
$x_p = \argmax_{x \in D} \widehat{K}(x,D)$ \tcp*{Seed}
$x_q = \argmax_{x \in D \setminus \{x_p\}} \widehat{K}(x, \{x_p\})$\;
$\gamma = (1-\varrho) \times \widehat{K}(x_q, \{x_p\})$\;
\If{$\gamma \le \tau$}{Terminate while-do loop\;}
$k$++\;
$G_0^k = \{x_p, x_q \}$ \tcp*{Initial cluster $k$}
\For{($i=1;\ \gamma  > \tau;\ i$++)}{
  $G_i^k = \{ x \in D\ |\ \widehat{K}(x, G_{i-1}^k) > \gamma \}$\; 
  $\gamma = (1-\varrho) \gamma$\; 
  }
  $G^k=G_{i-1}^k$ \tcp*{Cluster $k$ grown}
  $D= D \setminus G^k$\;
}
$N = D$\;
\Return $G^j, j=1,\dots,k;\ N$\;
 \caption{point-set Kernel Clustering \texttt{psKC}}
 \label{alg:pset-KC}
\end{algorithm}

Here we formally define  the cluster which is grown from a seed, according to \texttt{psKC}.

\begin{definition}
A $\widehat{K}^\tau_\varrho$-expanded cluster grows from a seed $x_p$ selected from $D$, using $\widehat{K}(\cdot, \cdot)$ with similarity  threshold $\tau < 1$ and growth rate $\varrho \in (0,1)$, is defined recursively as: 
\[
G_i = \{ x \in D\ |\ \widehat{K}(x, G_{i-1}) > \gamma_{i}  >\tau \}
\]
where $x_q = \argmax_{x \in D \setminus \{x_p\}} \widehat{K}(x, \{x_p\})$; $G_0 = \{ x_p, x_q\}$; $\gamma_i = (1-\varrho) \gamma_{i-1}$; and $\gamma_0 = (1-\varrho)\widehat{K}(x_q, \{x_p\}) $. 

\label{def:ExpandedCluster}
\end{definition}

The number of iterations required to expand each cluster up to the threshold $\tau$ is shown in lines 11-14 in Algorithm~\ref{alg:pset-KC} which are the key part of the algorithm. The rest of the algorithm is for initialization and house keeping only.

Let $G^j$ be $\widehat{K}^\tau_\varrho$-expanded cluster $j$ from $D$. 

The number of $\widehat{K}^\tau_\varrho$-expanded clusters in $D$ is discovered automatically by repeating the above cluster growing process on $G^k$ from $D \setminus \{G^j, j=1,\dots,k-1 \}$.

\begin{figure*}[h]
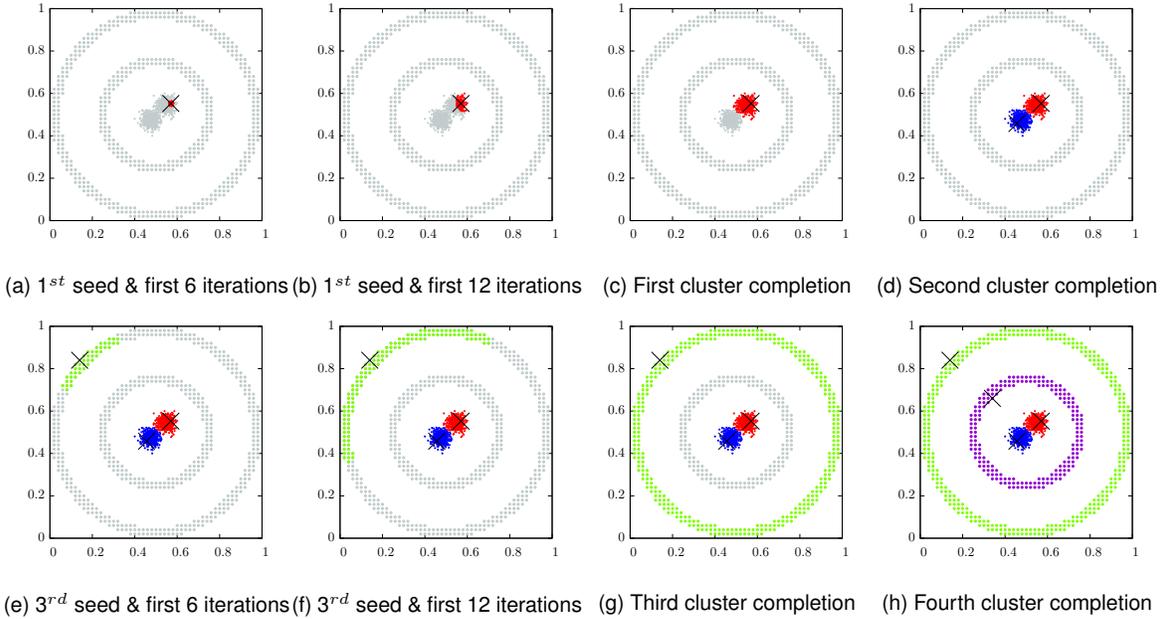

\vspace{-5mm}
 \centering
 \subfloat[1$^{st}$ seed \& first 6 iterations]{\includecombinedgraphics[scale=0.5,vecfile=2U_ring_2G_1st_cls-eps-converted-to]{2U_ring_2G_1st_cls}}
 \subfloat[1$^{st}$ seed \& first 12 iterations]{\includecombinedgraphics[scale=0.5,vecfile=2U_ring_2G_1st_cls_2nd-eps-converted-to]{2U_ring_2G_1st_cls_2nd}}
 \subfloat[First cluster completion ]{\includecombinedgraphics[scale=0.5,vecfile=2U_ring_2G_1st_cls_all-eps-converted-to]{2U_ring_2G_1st_cls_all}}
 \subfloat[Second cluster completion ]{\includecombinedgraphics[scale=0.5,vecfile=2U_ring_2G_2nd_cls_all-eps-converted-to]{2U_ring_2G_2nd_cls_all}} \\
 \vspace{-4mm}
 \subfloat[3$^{rd}$ seed \& first 6 iterations]{\includecombinedgraphics[scale=0.5,vecfile=2U_ring_2G_3rd_cls-eps-converted-to]{2U_ring_2G_3rd_cls}}
 \subfloat[3$^{rd}$ seed \& first 12 iterations]{\includecombinedgraphics[scale=0.5,vecfile=2U_ring_2G_3rd_cls_2nd-eps-converted-to]{2U_ring_2G_3rd_cls_2nd}}
 \subfloat[Third cluster completion]{\includecombinedgraphics[scale=0.5,vecfile=2U_ring_2G_3rd_cls_all-eps-converted-to]{2U_ring_2G_3rd_cls_all}}
 \subfloat[Fourth cluster completion]{\includecombinedgraphics[scale=0.5,vecfile=2U_ring_2G_4th_cls-eps-converted-to]{2U_ring_2G_4th_cls}} \\

  \caption{\texttt{psKC} clustering outcome of the Ring-G dataset. Crosses indicate the seed points. Gray points are data points in $D$ which are yet to be clustered. The number of iterations $f$ refers to that in lines 11-14 in Algorithm \ref{alg:pset-KC}. $f$ used to recruit all members of  each cluster is: 20, 19, 20 \& 20.}
  \label{fig:ring-G}
\end{figure*}

\begin{definition} After discovering all $\widehat{K}^\tau_\varrho$-expanded clusters $G^{j}$ in $D$,
noise is defined as
\[
N = \{ x \in D\ |\  \forall_j\ \widehat{K}(x, G^{j}) \le  \tau \}.
\]
\end{definition}

A post-processing is applied to all clusters produced by \texttt{psKC} to ensure that the following objective is achieved:
\begin{equation}
\Gamma(D) = \max_{G^1,\dots,G^k}  \sum^k_{j=1} \sum_{x \in G^j} \left< \Phi(x), \widehat{\Phi}(G^j) \right>.
\label{eqn_maximzing-dot-product}
\end{equation}

This post-processing reexamines all points which have the lowest similarity wrt cluster $G^j$ if they could be reassigned to other cluster to maximize the total similarity. 

A demonstration of the clustering process of \texttt{psKC} is shown in Figure \ref{fig:ring-G}, based on a two-dimensional dataset which has data points distributed in two concentrical rings and two Gaussian clusters. Figure~\ref{fig:ring-G}(a) shows the progression of identifying the first seed and growing the first cluster to include all members found in the first 6 iterations; followed by including all members found in the first 12 iterations in Figure~\ref{fig:ring-G}(b); and all members in the first cluster are found in Figure~\ref{fig:ring-G}(c)---this is when $\widehat{K}(x, G^1) < \tau$ for all $x$ in $D$ excluding $G^1$. The same process in identifying the third cluster, i.e., the outer ring, is shown in the first three subfigures in the second row of Figure~\ref{fig:ring-G}.

The time complexity of \texttt{psKC} is given in Table \ref{tab:time_complexity}. Line 12 in the for-loop iteration  in Algorithm \ref{alg:pset-KC} is the most expensive part. $\widehat{K}$ in line 12 needs to be evaluated $n-2$ times in the very first iteration; and any subsequent iterations take monotonically decreasing number of $\widehat{K}$ evaluations.
In the worst case,
$f(\tau,\varrho)$ is the maximum number of iterations. 

\begin{proposition}
\texttt{psKC} performs a maximum number of iterations $f$ in the cluster expansion process for each cluster, which is independent of data size, that is a function of $\tau$ and $\varrho$: \[f(\tau,\varrho) = \biggl\lfloor  \frac{\log{\tau}}{\log{(1-\varrho)}} \biggr\rfloor .\] 
\end{proposition}

\begin{proof}
By setting the initial $\gamma=1$ to its maximum value before the iteration in lines 11-14, the maximum number of iterations $f$ is reached when $(1-\varrho)^f = \tau$. As it is independent of the data size, this provides the proof. 
\end{proof}

This is a very useful property in practice because one can use it to bound the number of iterations. For example, the settings of $\tau \in [10^{-5},10^{-1}]$ and $\varrho \in [0.1,0.26]$, that we have used in the experiments reported in Section \ref{sec_experiment}, bound $f \in [8,109]$.

In comparison, the time complexity of k-means is superpolynomial in the worst case which requires  $f'(n)=2^{\Omega{(\sqrt{n})}}$ iterations \cite{k-means-SCG2006}, where its time complexity is $O(f'(n) nkd)$ \cite{k-means-1979}. Thus, the key difference is the number of iterations, i.e., \texttt{psKC} controls $f(\tau,\varrho)$ via parameter setting which is independent of $n$; k-means has no such mechanism and $f'(n)$ is a function of $n$. In practice, \texttt{psKC} has its iterations bounded by the search ranges of $\tau$ and $\varrho$, as we have shown earlier, giving it a linear time complexity. Though $f'(n)$ is typically a small fraction of $n$ in practice, no proof has been provided thus far.

\begin{table}[h]
 \centering
 \caption{Time complexities of \texttt{psKC} and k-means.\protect\linebreak $t$ and $\psi$ are parameters of Isolation Kernel; $d$ is number of dimensions.}
\label{tab:time_complexity}
 \begin{tabular}{@{\ }lp{5.5cm}@{}r@{\ }}
  \toprule
1 & Build Isolation Kernel (IK) & $dt\psi$\\
2 & Mapping $D$ of $n$ points using feature map of IK & $n dt\psi$ \\ 
3 & $f(\tau,\varrho)$ iterations in lines 11-14 for $k$ clusters  & $f(\tau,\varrho) n kdt \psi$\\ 
\multicolumn{2}{l}{Total time cost for \texttt{psKC}}  & $f(\tau,\varrho) nkdt \psi$\\ \midrule
\multicolumn{2}{l}{k-means, $f'(n)$ is typically a small fraction of $n$ in practice} 
& $f'(n) nkd$\\
  \bottomrule
 \end{tabular}
\end{table}

\section{Conceptual differences in comparison with existing clustering algorithms}
\label{sec_conceptual-diff}
The proposed clustering algorithm \texttt{psKC} is unique among the existing clustering algorithms in two key aspects:
\begin{itemize}
    \item It is the only clustering algorithm which utilizes the \textbf{point-set kernel that employs Isolation Kernel}.
    \item The algorithmic design is unique: the main step in the procedure employs the proposed point-set kernel to grow every cluster from a seed.
\end{itemize}

Both the point-set kernel and the procedure are crucial in achieving a clustering algorithm which is both effective and efficient. 

The finite dimensional feature map of Isolation Kernel and its use in the point-set kernel enable the algorithm to achieve its full potential: runtime proportional to data size ($n$)---a level unable to be achieved by existing effective clustering algorithms such as DP \cite{DensityPeak-2014}, and even less effective but efficient algorithms such as scalable kernel k-means \cite{Scalable-kMeans-JMLR19}. They have at least $n^2$ runtimes.

Existing clustering algorithms have the following features:
 \begin{itemize}
    \item[(a)] Density-based clustering algorithms such as DP and DBSCAN rely on point based density estimation that in turn relies on point-to-point distance calculations. This is the key reason why they both have $n^2$ runtime.  
     \item[(b)] Kernel k-means \cite{Scalable-kMeans-JMLR19,KNNKernel}) is weaker  than \texttt{psKC} for two reasons. First, the clustering algorithm employed, i.e., k-means, produces clusters that are limited to globular shape and approximately the same size \cite{AgarwalDMBook2015}. These limitations are inherited in feature space with the use of kernel. That is the reason why kernel k-means could not identify clusters of arbitrary shapes and different sizes (see Table \ref{tab:compare} for details). Second, the use of a kernel which has intractable dimensionality and is data independent. The intractable dimensionality issue necessitates a kernel functional approximation \cite{Nystrom_NIPS2000} in order to produce a finite-dimensional feature map. This approximation  reduces the quality of the final clustering outcome. The data independence issue (often ignored/unrecognized) also leads to poorer clustering outcomes.
 \end{itemize}
 
It is possible to view the centers in kernel k-means as a kind of kernel mean map (defined in Equation \ref{eqn_kernel-mean-map}) because each    cluster center is defined as the average position of all points in a cluster in the feature space. However, both limitations mentioned above have constricted the types of clusters it can detect. 

In a nutshell, many existing clustering algorithms compute point-to-point distance/kernel to derive the required similarity to characterize clusters. Even in algorithms that can be considered to have used point-set kernel, their cluster characterizations have severe limitations.  In contrast, the proposed algorithm utilizes the point-set kernel built from Isolation Kernel. This allows clusters of arbitrary shapes, sizes and densities to be characterized successfully and efficiently.

A more detailed comparison with kernel k-means is instructive.
Kernel k-means uses the following objective function:
\begin{equation}
\argmin_{G^1,\dots,G^k} \frac{1}{n} \sum^k_{j=1} \sum_{x \in G^j} \parallel \Phi(x) -  \widehat{\Phi}(G^j) \parallel^2_2
\label{eqn_minimzing-sum-squared}
\end{equation}

The difference in objective functions between \texttt{psKC} and kernel k-means is small (between Equations \ref{eqn_maximzing-dot-product} and \ref{eqn_minimzing-sum-squared}), i.e., the latter uses a squared distance; and the former employs a dot product.

The key differences are the algorithms used to achieve these objectives.
The k-means procedure iterates over two main steps: the center computation for each group; and the reassignment of every point to its nearest center. The simplicity is k-means's advantage in time complexity over other existing algorithms. But it also gives rise to its shortcomings. First, the initialization is a random grouping of clusters, and the number of clusters is specified by the user. The initial random grouping could lead to (a) counter-intuitive clustering outcomes even when the number of clusters is selected correctly; and (b) the clustering results are unstable from one run to the next. This means that k-means can easily get trap in local minimum. Second, the procedure always produces spherical clusters of similar size in the feature space of a kernel. Third, it recomputes distances to centers for many points which do not change the assignment in the previous iteration. This constitutes a large expense of the computational cost which adds no value but is necessary in the procedure.  Fourth, the use of kernel mean map to characterize each center makes k-means sensitive to outliers. In other words, the above shortcomings of kernel k-means are algorithmic, not due to a specific kernel used.

The \texttt{psKC} procedure does not have these shortcomings. First, the algorithm is deterministic, given a kernel function and the user-specified parameters. This resolves the instability issue and often leads to better clustering outcomes.  The only randomisation is due to the Isolation Kernel. The use of most similar points in $D$ as seeds is much more stable, even with different initializations of Isolation Kernel, compare with random groupings of clusters which can change wildly from one run to the next. Second, \texttt{psKC} enables detection of clusters of arbitrary shape, of different sizes and densities. Third, \texttt{psKC} commits  each point to a cluster once it is assigned; and  most points which are similar to the cluster never need to be reassigned. This is possible because of the use of a seed to grow a cluster. Points which are similar to a cluster grown from the seed will not be similar to another cluster if the points are less similar to the seeds of other clusters in the first place. The sequential determination of seeds (as opposed to the parallel determination of centers in k-means) makes that possible. As a result, \texttt{psKC} avoids many unnecessary recomputations in k-means mentioned earlier. Fourth, in contrast to k-means, the use of kernel mean map, not as a cluster center, but as a medium to expand a growing cluster  in \texttt{psKC} makes it robust to outliers. This is because outliers are least similar to a growing cluster; they are thus unlike to be included in the cluster.

There are two other important differences: \textbf{\texttt{psKC} is not an optimization algorithm}; whereas k-means is an expectation-and-maximization optimization method; and yet the clustering outcome of \texttt{psKC} is already close to the final maximization objective. The post-processing, literally tweaks at the edges, by reexamining those lowest similarity points wrt each cluster for possible better reassignments.

The relationship of spectral clustering and kernel k-means is given in Appendix \ref{kernel-k-means-other-issues}.

We show in the next section that the proposed clustering algorithm \texttt{psKC} is both highly efficient and producing good clustering outcomes in an empirical evaluation.

\setlength{\fboxrule}{2pt}
\begin{table*}
 \centering
  \caption{Artificial datasets: Clustering outcomes of \texttt{psKC}, DP, two versions of kernel k-means and one recent k-means-based SGL \cite{StructureGraphLearning-2021}. \protect\linebreak
  The results with yellow frames indicate good clustering outcomes; and those without have poor clustering outcomes.} \label{tab:compare}
 \begin{tabular}{@{}c@{}c@{}c@{}c@{}c@{}c}
  \toprule
Dataset & \texttt{psKC} & DP & \multicolumn{2}{c}{Kernel k-means} & SGL (k-means-based)\\ \cline{4-5}
 & Isolation Kernel & Distance & Gaussian kernel & kNN kernel & Distance\\
  \midrule
\raisebox{5.0\normalbaselineskip}[0pt][0pt]{Ring-G} & \fcolorbox{yellow}{white}{\includecombinedgraphics[scale=0.38,vecfile=2U_ring_2G_4th_cls-eps-converted-to]{2U_ring_2G_4th_cls}} & \includecombinedgraphics[scale=0.38,vecfile=2U_ring_2G_DP-eps-converted-to]{2U_ring_2G_DP} & \includecombinedgraphics[scale=0.38,vecfile=2U_ring_2G_k_kmeans-eps-converted-to]{2U_ring_2G_k_kmeans} & \includecombinedgraphics[scale=0.38,vecfile=2U_ring_2G_kNN-eps-converted-to]{2U_ring_2G_kNN} & \includecombinedgraphics[scale=0.38,vecfile=2U_ring_2G_SGL-eps-converted-to]{2U_ring_2G_SGL} \\
\raisebox{5.0\normalbaselineskip}[0pt][0pt]{AC} & \fcolorbox{yellow}{white}{\includecombinedgraphics[scale=0.38,vecfile=2C_new_mkVI-eps-converted-to]{2C_new_mkVI}} & \fcolorbox{yellow}{white}{\includecombinedgraphics[scale=0.38,vecfile=dp_AC_mkVI-eps-converted-to]{dp_AC_mkVI}} & \includecombinedgraphics[scale=0.38,vecfile=k_nys_2C_new_mkVI-eps-converted-to]{k_nys_2C_new_mkVI} & \fcolorbox{yellow}{white}{\includecombinedgraphics[scale=0.38,vecfile=kNN_AC_mkVI-eps-converted-to]{kNN_AC_mkVI}} &
\includecombinedgraphics[scale=0.38,vecfile=SGL_AC_mkVI-eps-converted-to]{SGL_AC_mkVI}\\ 
\raisebox{5.0\normalbaselineskip}[0pt][0pt]{Aggreg.} & \fcolorbox{yellow}{white}{\includecombinedgraphics[scale=0.38,vecfile=aggregation-eps-converted-to]{aggregation}}  & \fcolorbox{yellow}{white}{\includecombinedgraphics[scale=0.38,vecfile=dp_aggregation-eps-converted-to]{dp_aggregation}}  & \includecombinedgraphics[scale=0.38,vecfile=k_nys_aggregation-eps-converted-to]{k_nys_aggregation} & \includecombinedgraphics[scale=0.38,vecfile=kNN_aggregation-eps-converted-to]{kNN_aggregation} & 
\includecombinedgraphics[scale=0.38,vecfile=SGL_aggregation-eps-converted-to]{SGL_aggregation}\\
\raisebox{5.0\normalbaselineskip}[0pt][0pt]{Spiral} & \fcolorbox{yellow}{white}{\includecombinedgraphics[scale=0.38,vecfile=spiral-eps-converted-to]{spiral}}  & \fcolorbox{yellow}{white}{\includecombinedgraphics[scale=0.38,vecfile=dp_spiral-eps-converted-to]{dp_spiral}}  & \includecombinedgraphics[scale=0.38,vecfile=k_nys_spiral-eps-converted-to]{k_nys_spiral} & \includecombinedgraphics[scale=0.38,vecfile=kNN_spiral-eps-converted-to]{kNN_spiral} & 
\includecombinedgraphics[scale=0.38,vecfile=SGL_spiral-eps-converted-to]{SGL_spiral}\\
\raisebox{5.0\normalbaselineskip}[0pt][0pt]{S3} & \fcolorbox{yellow}{white}{\includecombinedgraphics[scale=0.38,vecfile=s3_15_cls-eps-converted-to]{s3_15_cls}}  & \fcolorbox{yellow}{white}{\includecombinedgraphics[scale=0.38,vecfile=dp_s3-eps-converted-to]{dp_s3}}  & \fcolorbox{yellow}{white}{\includecombinedgraphics[scale=0.38,vecfile=k_nys_s3-eps-converted-to]{k_nys_s3}} & \fcolorbox{yellow}{white}{\includecombinedgraphics[scale=0.38,vecfile=kNN_s3-eps-converted-to]{kNN_s3}} &
\fcolorbox{yellow}{white}{\includecombinedgraphics[scale=0.38,vecfile=SGL_s3-eps-converted-to]{SGL_s3}}\\
\bottomrule
 \end{tabular}

\end{table*}

\section{Empirical evaluation}
\label{sec_experiment}
We empirically compare \texttt{psKC} with DP \cite{DensityPeak-2014}, scalable kernel k-means \cite{Scalable-kMeans-JMLR19} which employs Gaussian kernel,  kernel k-means which employs adaptive kernel \cite{KNNKernel} and a recent k-means-based algorithm SGL \cite{StructureGraphLearning-2021} in terms of clustering outcomes and runtimes.

Commonly used benchmark datasets as well as color images are used in the experiments. As all these datasets can be visualized, we present the clustering outcomes of the algorithms under comparison by showing their segmented images or two-dimensional plots for visual inspection. The runtime is measured in terms of CPU seconds (and include the GPU seconds when GPU is used.)
Color images are represented in the CIELAB color space\footnote{In three dimensions as defined by the International Commission on Illumination (http://www.cie.co.at/).}. 
Other experimental settings can be found in Appendix \ref{App_Addtional_result}.

We present the results in four subsections: The first reports the clustering outcomes; the second presents the runtime comparison; the third provides the result of a stability analysis; and the last examines the effect of the post-processing.

\subsection{Clustering outcomes}
\label{sec_clustering_outcomes}
The clustering outcomes of artificial datasets, image segmentation and finding groups of handwritten digits are presented in three separate sections below.
\subsubsection{Artificial datasets}
A comparison of clustering outcomes of five clustering algorithms on five benchmark datasets is provided in Table \ref{tab:compare}. It is interesting to note that DP did well in four benchmark datasets, but it did poorly on the Ring-G dataset, i.e., DP successfully identified one ring cluster and one Gaussian cluster; but split the second Gaussian cluster into two parts, where one part is joined with the second ring cluster. This is because DP has a weakness in identifying the correct peaks when (some) clusters are uniformly distributed and have varied densities. In this case, three out of the four peaks are identified in the two Gaussian clusters before other points are assigned to one of the peaks to form clusters.

All three k-means-based algorithms are weaker than DP as they did poorly on at least three out of the five datasets, i.e., Ring-G, Aggregation and Spiral. This is because of their use of k-means which has weaknesses in detecting clusters that have non-globular shapes \cite{Tan-2ED-2018}. The use of a kernel in both kernel k-means transfers these fundamental weaknesses from input space to feature space. The results in Table \ref{tab:compare} show that there is no guarantee that they can detect clusters of non-globular shapes in input space.

\setlength{\tabcolsep}{2pt}
\begin{table}
 \centering
  \caption{Clustering outcomes of \texttt{psKC} \& SGL on  Zhao Mengfu's Autumn Colors.}
 \label{fig:AutumnColors}
 \begin{tabular}{ccc}
  \toprule
\begin{turn}{90}  \hspace{2mm}{CIELAB Space}    
\end{turn} & 
 \subfloat[\texttt{psKC}]{\includegraphics[scale=0.07]{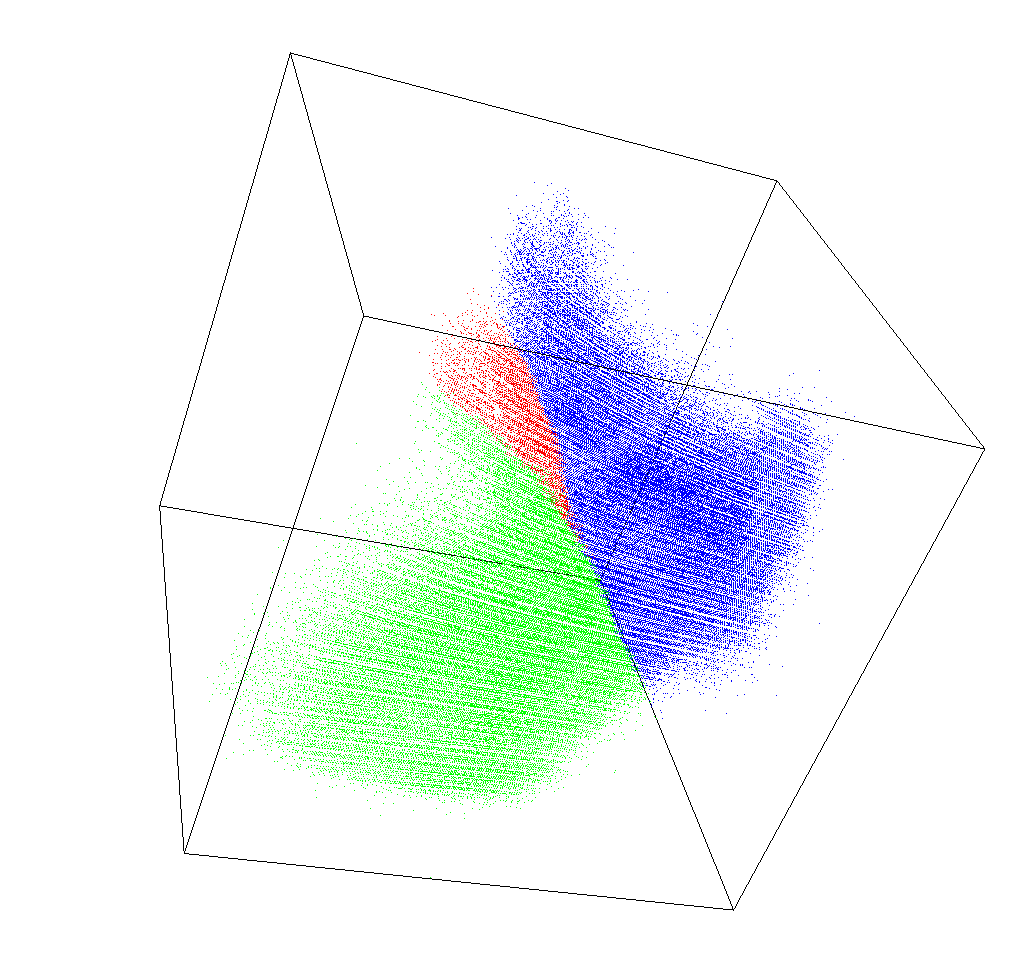}}  \subfloat[SGL]{\includegraphics[scale=0.09]{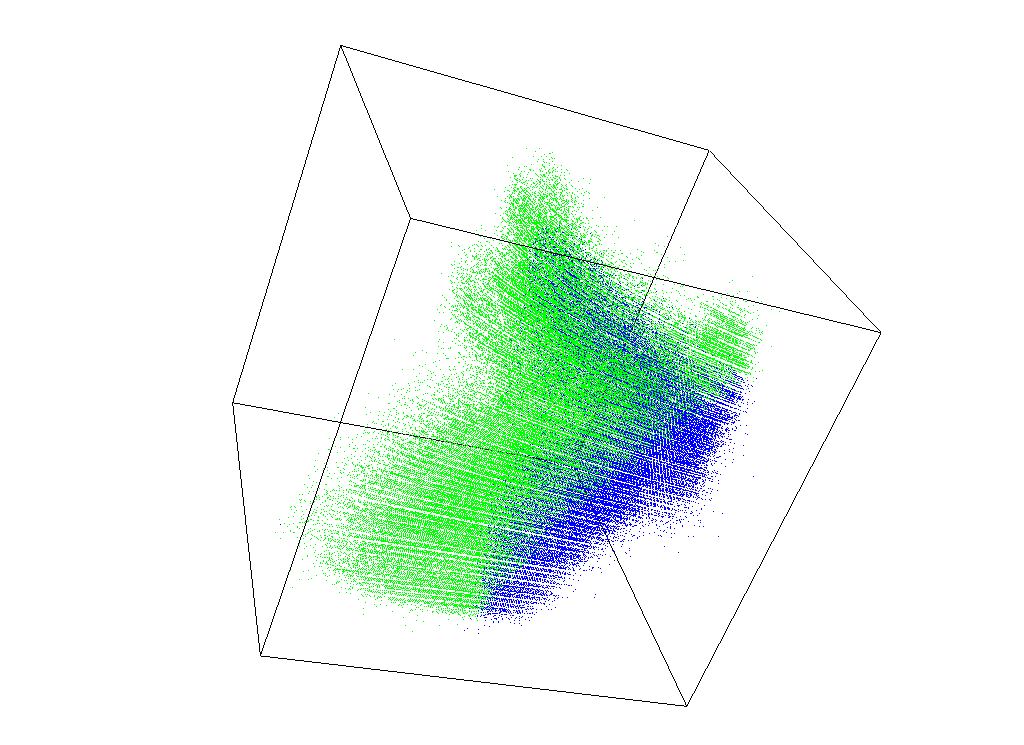}} \\
\begin{turn}{90}  \hspace{3.mm}{Original}  \end{turn}   & 
 \includegraphics[scale=0.1 ]{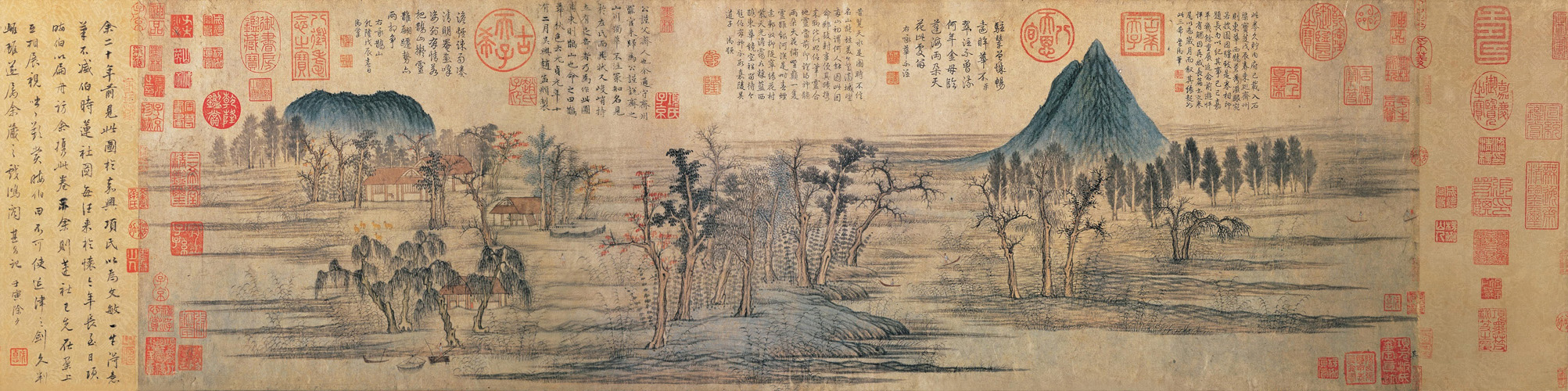}\\ \midrule
\begin{turn}{90}  {\texttt{psKC}} 
\end{turn} &  \includegraphics[scale=0.1 ]{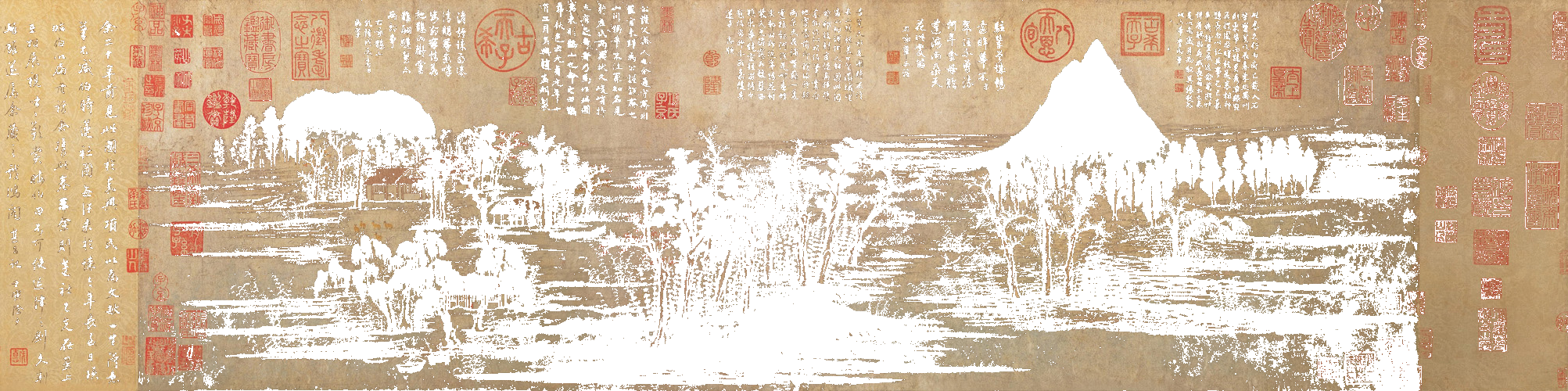}\\
 &  
\includegraphics[scale=0.1 ]{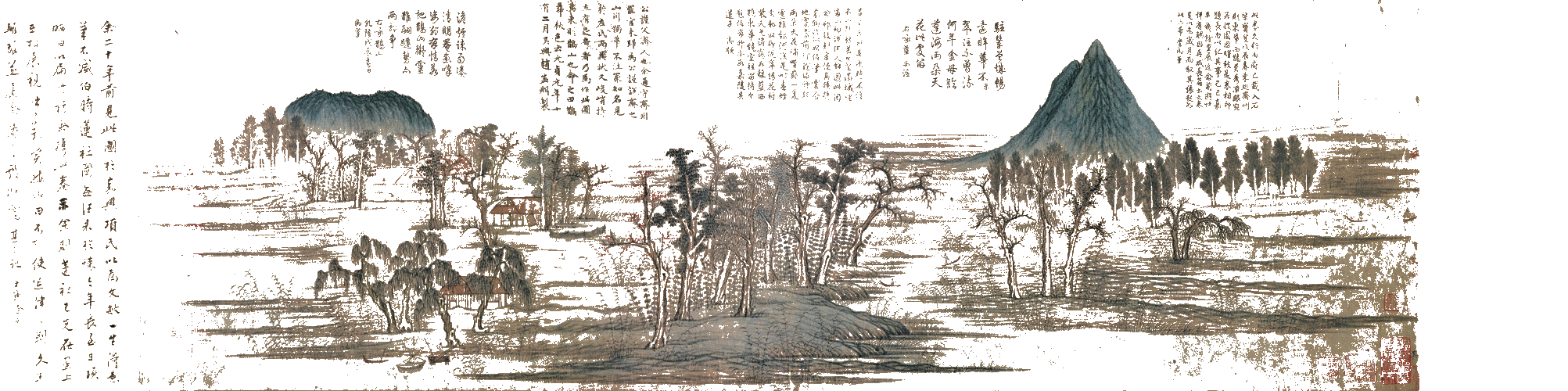} \\ 
 &  
\includegraphics[scale=0.1 ]{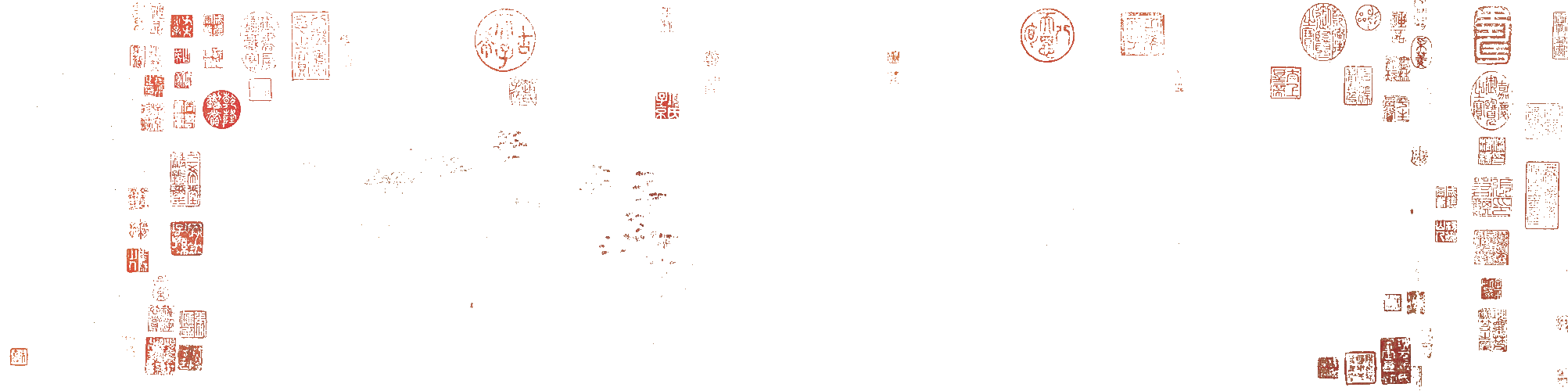} \\ 
\midrule
\begin{turn}{90}  {SGL} 
\end{turn} &  \includegraphics[width=0.38\textwidth]{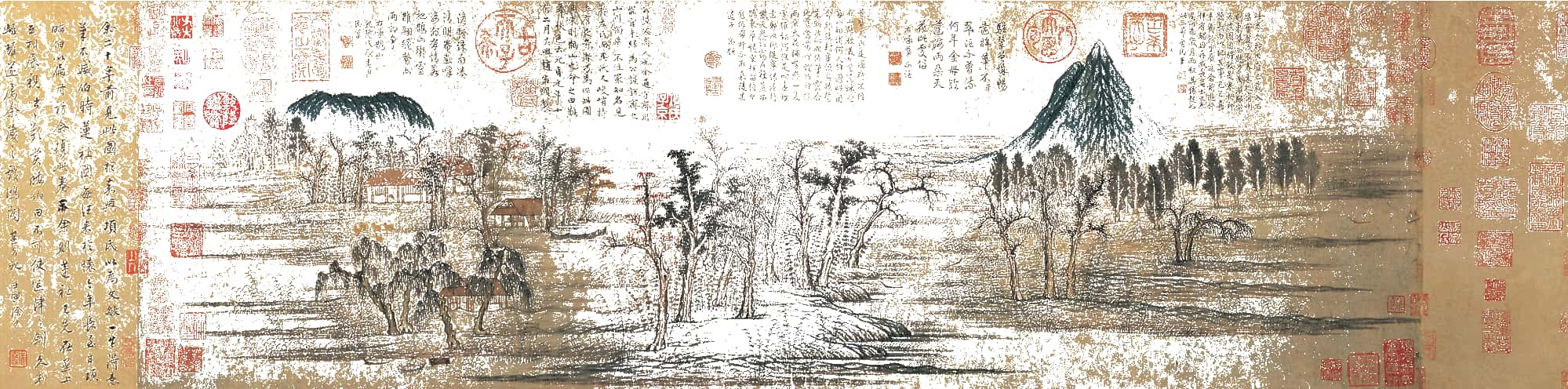}\\
 &  
\includegraphics[width=0.38\textwidth]{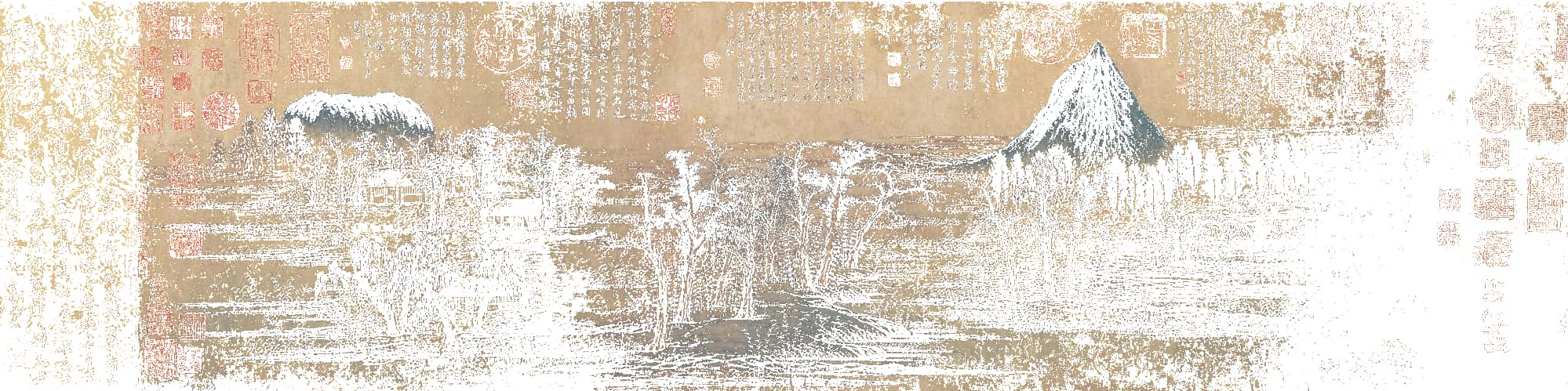} \\
\bottomrule
 \end{tabular}
\end{table}

\texttt{psKC} is the only algorithm that did well in all five datasets; and it is the only algorithm that successfully identified all four clusters in the Ring-G dataset. This is a direct result of the proposed cluster identification procedure which employs the point-set kernel. Other algorithms failed to correctly identify the four clusters because their algorithmic design which must determine all density peaks/centers before individual points can be assigned to one of the peaks/centers. 

The S3 dataset is the easiest to cluster. All five algorithms have good clustering outcomes though the outcomes differ slightly locally. The AC dataset is the second easiest. All algorithms, except kernel k-means with Gaussian kernel, have produced the perfect clustering outcome. Both Aggregation and Spiral were successfully clustered by \texttt{psKC} and DP; but both versions of kernel k-means  failed to separate all clusters correctly.

\subsubsection{Image segmentation}
Here we examine the ability of the five clustering algorithms in dealing with  images of high resolution.

\setlength{\tabcolsep}{0pt}
\begin{table}
 \centering
  \caption{Clustering outcomes of \texttt{psKC} and SGL on  Vincent van Gogh's Starry Night
  over the Rhone 2.}
 \label{fig:Starry-Night}
 \begin{tabular}{cc}
  \toprule
\begin{turn}{90}  \hspace{10mm}{Original}  \end{turn} \ \ \   &
 {\includegraphics[scale=0.11]{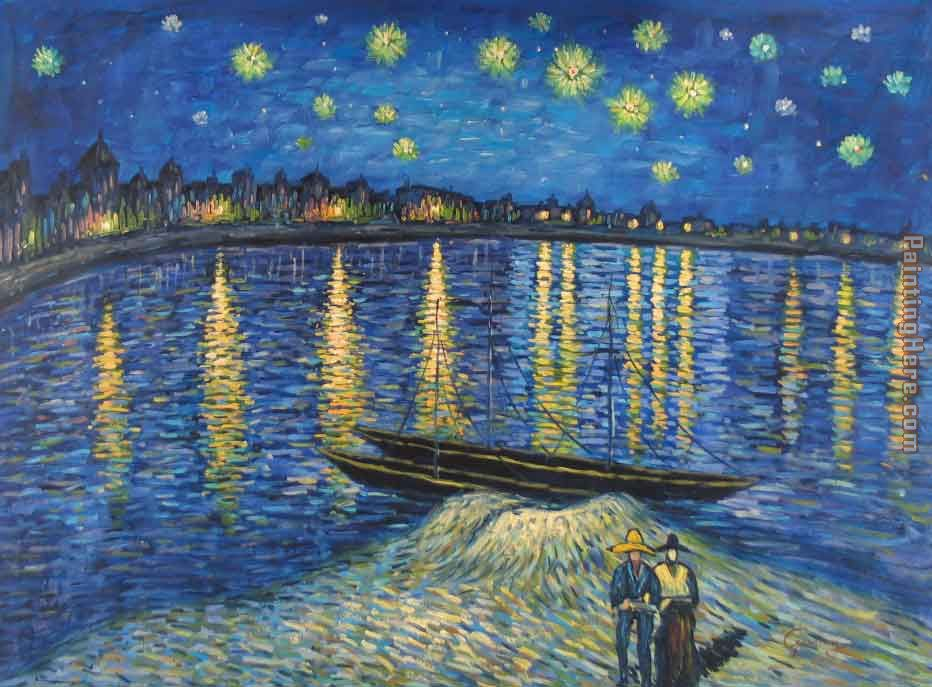}}
 \subfloat[ \texttt{psKC}]{\includegraphics[scale=0.08,trim={4cm 0cm 4cm 0cm},clip]{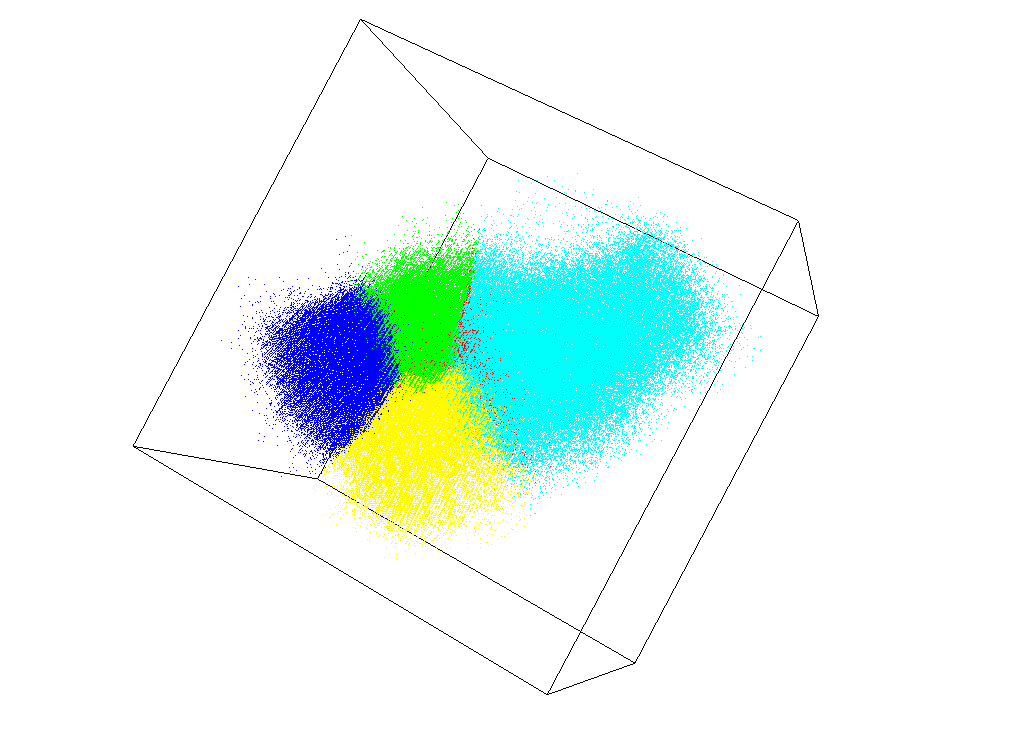}}
 \subfloat[SGL]{\includegraphics[scale=0.08,trim={5cm 0cm 5cm 0cm},clip]{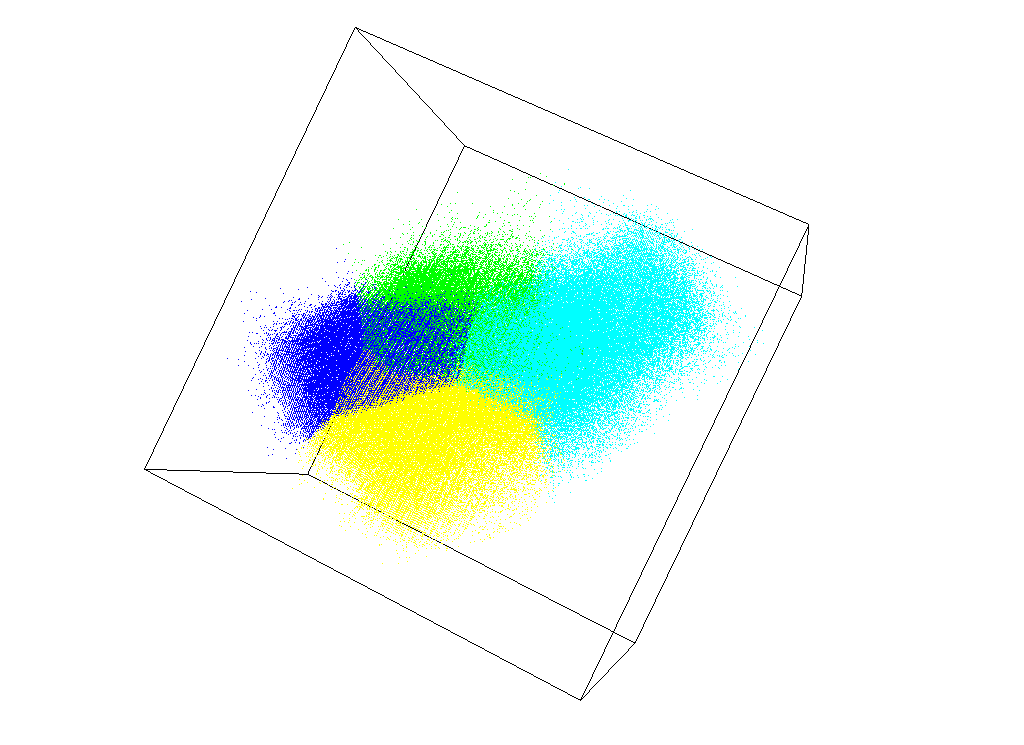}}
 \\ \midrule
\begin{turn}{90}  \hspace{3.mm}{\texttt{psKC}}  \end{turn}   & 
 \includegraphics[scale=0.11,trim={0cm 0cm 0cm 0cm},clip]{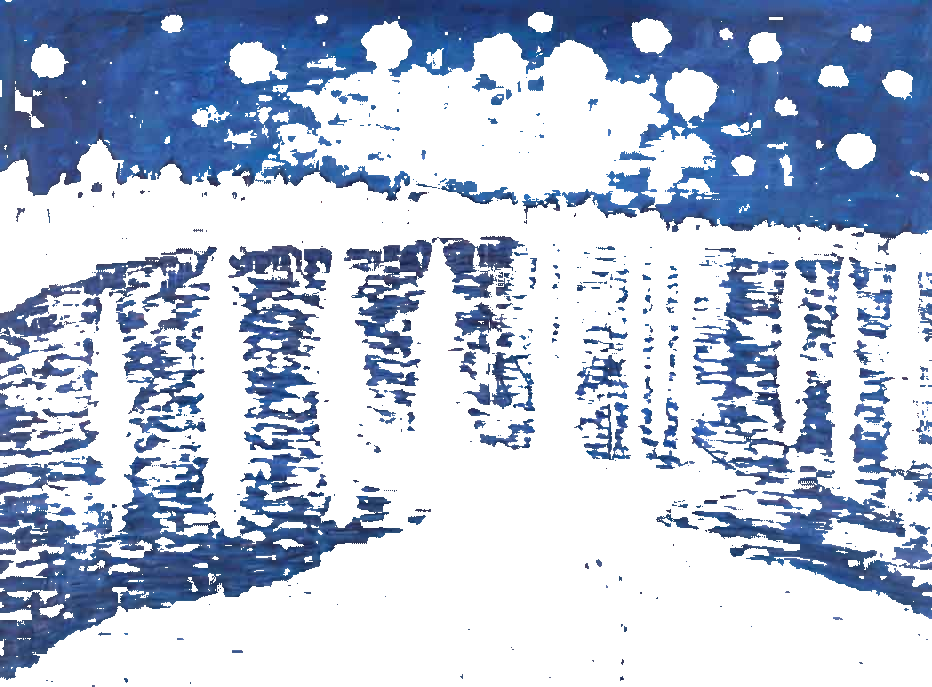}
 \includegraphics[scale=0.11,trim={0cm 0cm 0cm 0cm},clip]{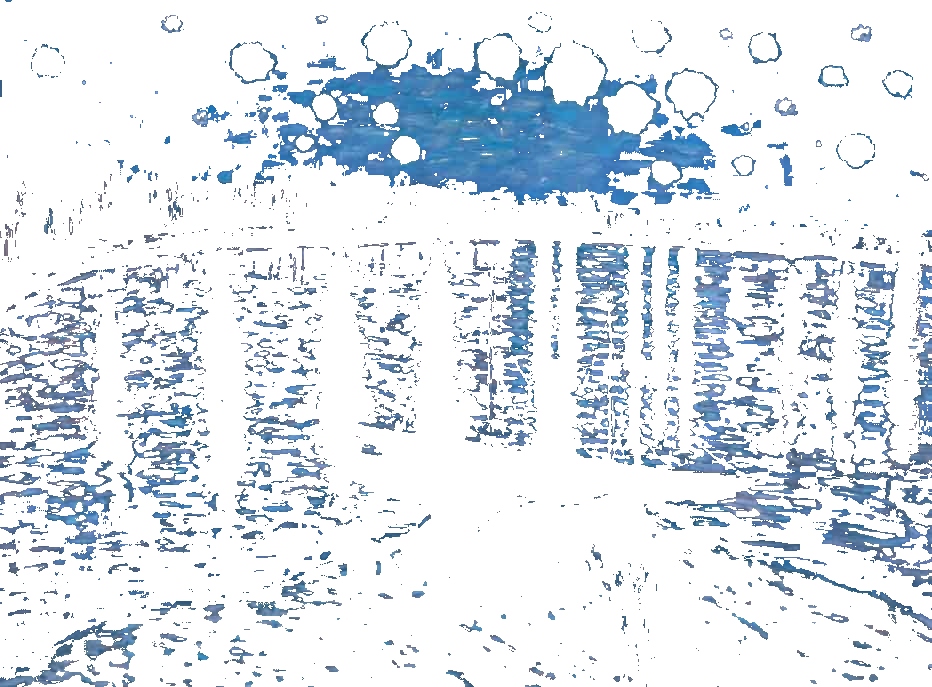}\\
 & \includegraphics[scale=0.11,trim={0cm 0cm 0cm 0cm},clip]{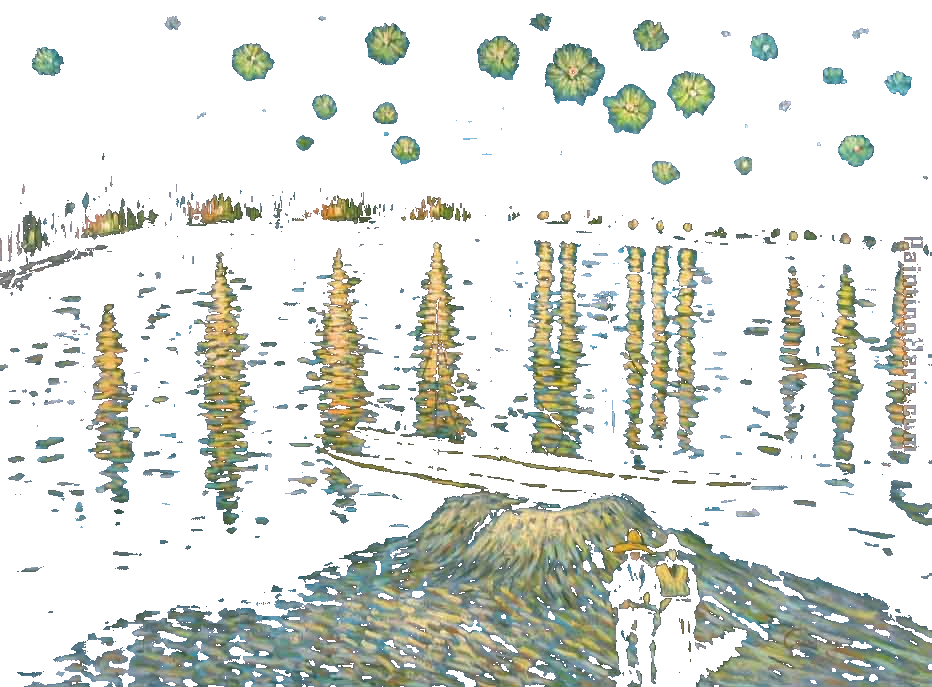} \includegraphics[scale=0.11,trim={0cm 0cm 0cm 0cm},clip]{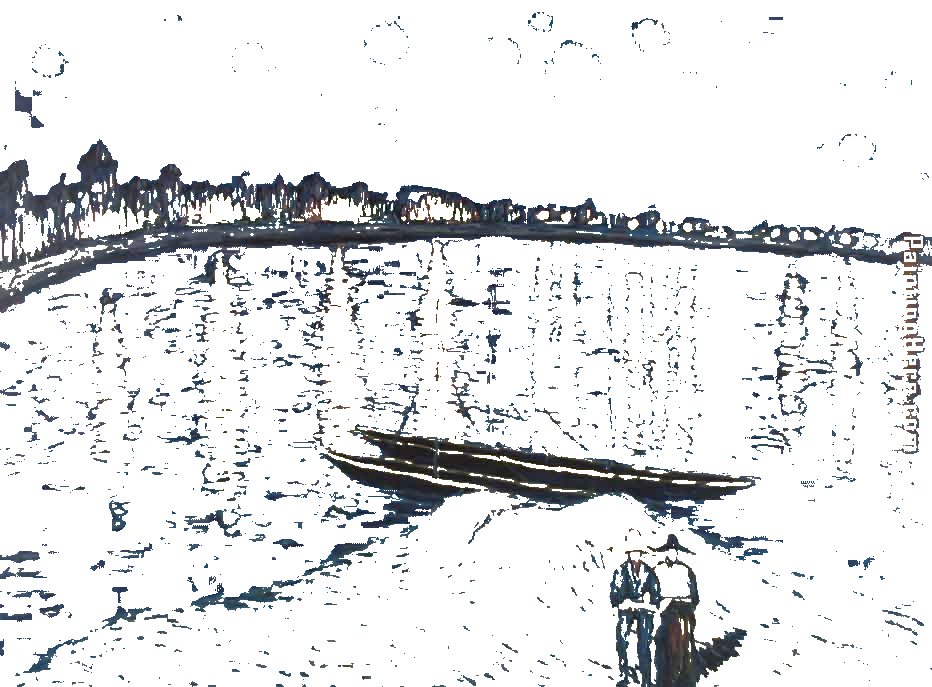}\\ \midrule
 \begin{turn}{90}  \hspace{3.mm}{SGL}  \end{turn}   & \includegraphics[width=0.2\textwidth]{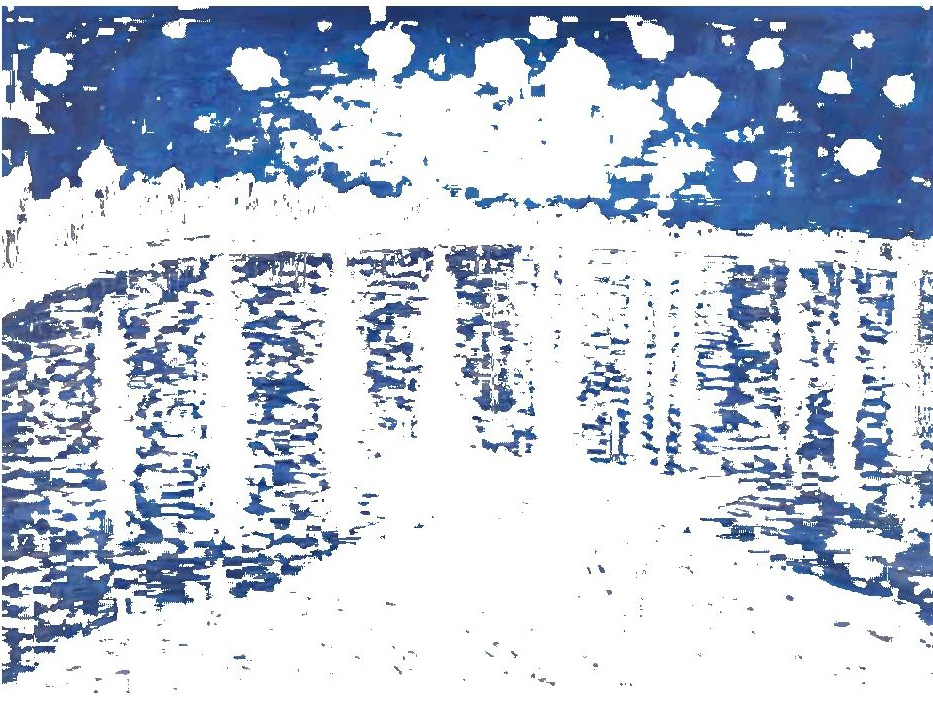}   \includegraphics[width=0.2\textwidth]{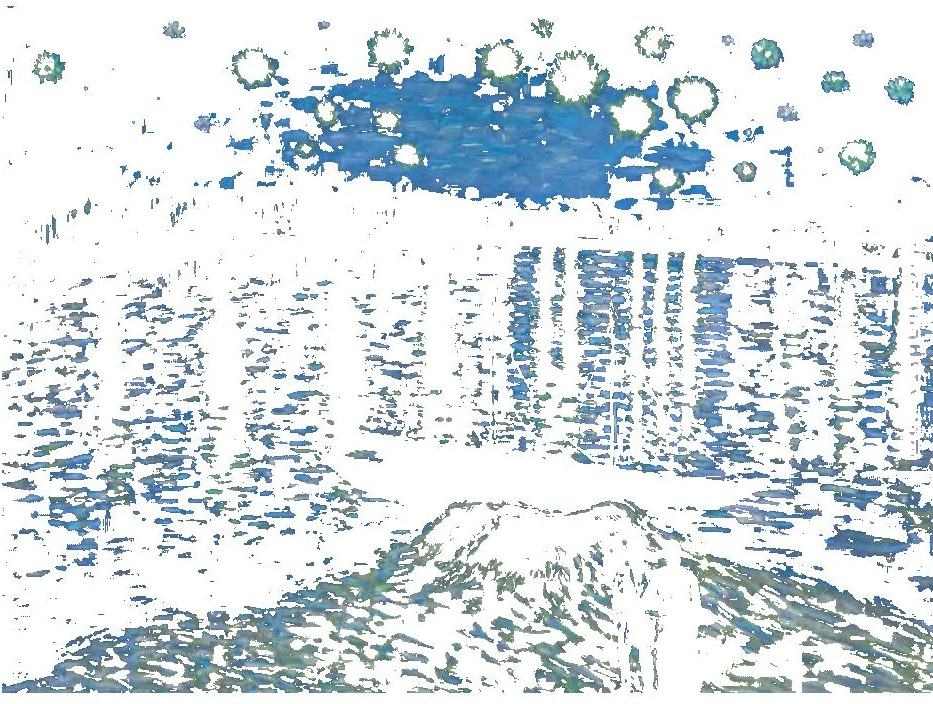}\\
 & \includegraphics[width=0.2\textwidth]{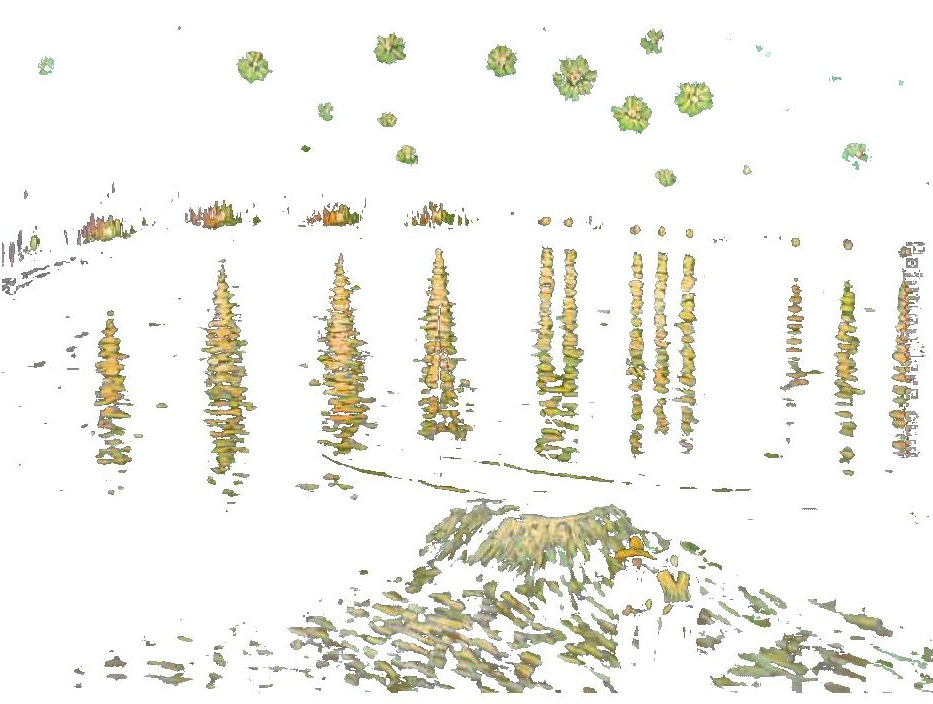}   \includegraphics[width=0.2\textwidth]{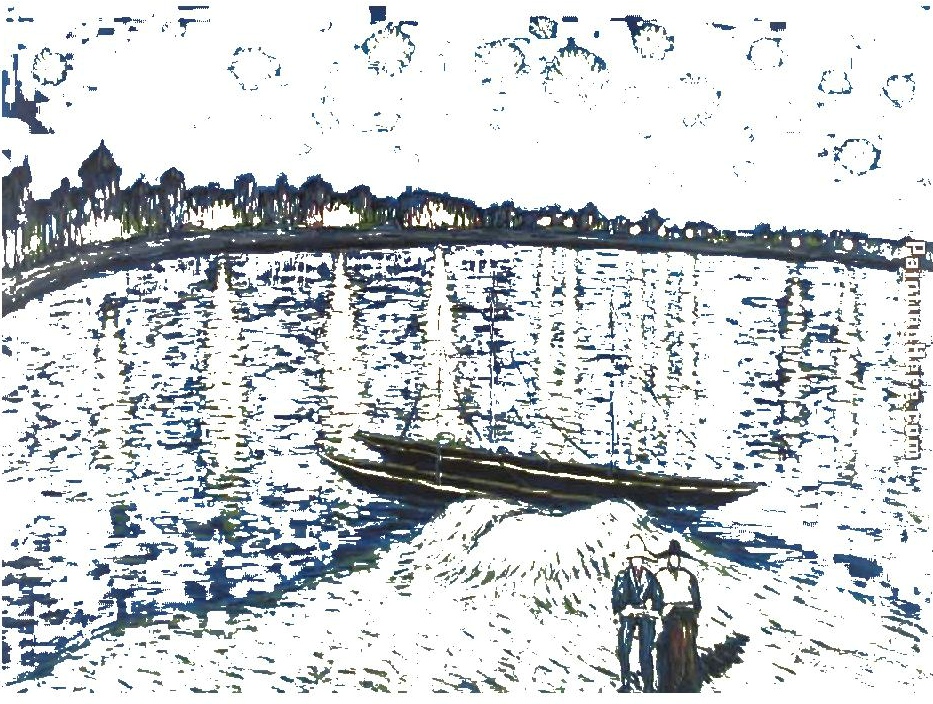}\\
 \bottomrule
 \end{tabular}
\end{table}

Out of the four contenders of \texttt{psKC}, only SGL could complete the run in reasonable time.

Table \ref{fig:AutumnColors}  shows the clustering outcomes of \texttt{psKC} and SGL on
Zhao Mengfu's Autumn Colors.
\texttt{psKC} separates the background (plus red stamps) of this painting from the landscape, producing two distinctive clusters. The red stamps can be extracted as a cluster on its own if a different parameter setting is used.

In contrast, SGL produces a clustering outcome which has less clear distinction between the background and landscape, and different portions of red stamps appear in separate clusters.

\texttt{psKC} and SGL have similar clustering outcomes on the van Gogh's painting, shown in Table \ref{fig:Starry-Night}.

None of DP and kernel k-means can complete in reasonable time on both images. Each of these images has a total of more than 1 million pixels. DP was unable to load the dataset on a machine with 256GB of main memory because of high memory requirement using Matlab. Scalable kernel k-means took $>4$ days (we terminated the run as it took too long to complete.) 

\subsubsection{Finding groups in a collection of handwritten images}

Here we show the groups of  handwritten digits found by \texttt{psKC} on the MNIST70k dataset which has 70,000 images.

The largest 13 clusters produced by \texttt{psKC} cover 92\% of the images in the dataset.
Table \ref{fig_mnist70k} shows  examples of these clusters in (a). Subfigure (b) in Table \ref{fig_mnist70k} shows the sample noise images which do not belong to large clusters. Notice that the noise digits have different handwritten styles from those in the large clusters.

\begin{table}
 \centering
  \caption{Clustering outcome of \texttt{psKC} and SGL on the MNIST70k data set. The columns from left to right show the most similar digit to the least similar digit from each of the 10 equal-frequency bins in each cluster, sorted by $\widehat{K}$ in Definition \ref{def:similarity_def} for \texttt{psKC} \& by distance for SGL.}
 \label{fig_mnist70k}
 \vspace{-4mm}
 \begin{tabular}{cc}
  \toprule
\begin{turn}{90}  \hspace{10mm}{\hspace{2cm}\texttt{psKC}}  \end{turn} \ \ \   &
 \subfloat[Largest 13 clusters]{\includegraphics[scale=0.50]{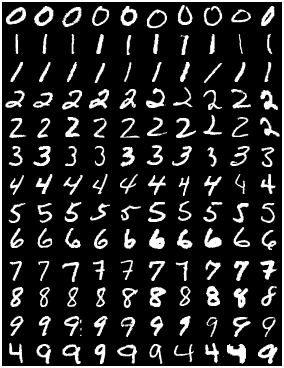}}
 \subfloat[Noise]{\makebox[2cm][c]{\includegraphics[scale=0.50]{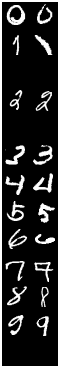}}}
 \\ \midrule
\begin{turn}{90}  \hspace{2cm}{SGL}  \end{turn} \ \ \   & 
\includegraphics[scale=0.5]{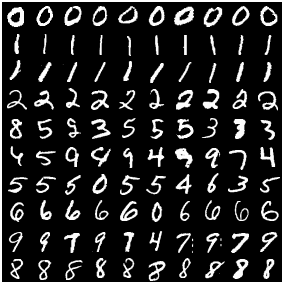}\\
\bottomrule
\end{tabular}
\end{table}

It is interesting to note that each of digits 1 \& 2 has been grouped into two clusters, where each cluster has its own written style, e.g., digit 1 has a vertical written style in the first cluster; and a slant style in the second cluster.

Also note that digits 4 \& 9 are grouped into three clusters. In addition of the two pure clusters, the third cluster  consists of both digits of 1:3 proportion. This is in contrast to the result produced by a kNN-graph based clustering algorithm RCC \footnote{RCC was unable to complete this task in 3 days on our machine;
whereas \texttt{psKC} ($\psi=5000$) took $<25$ minutes. 
See a brief description of RCC and its clustering outcomes of the artificial datasets in Appendix \ref{App_Addtional_result}.} (see Fig.3 in \cite{shah2017robust}), where both digits 4 \& 9 have been grouped into a single cluster. Also notice that the digits grouped in the third cluster have different written styles to those in the two pure clusters of 4 \& 9.

Table \ref{fig_mnist70k} also shows  the clustering outcome of SGL. Although SGL produces pure clusters for digits 0,1,2 and 8, it has mixed up multiple digits into individual clusters for other digits. Setting $k=13$ produced worse result. 

DP and the two versions of kernel k-means \cite{Scalable-kMeans-JMLR19,KNNKernel} were unable to complete this task in reasonable time.

\noindent
\textbf{Summary}: Only \texttt{psKC} and SGL could complete the clustering of all datasets/images used in the experiments. However, SGL produces poorer clustering outcomes than \texttt{psKC} on all datasets/images, except two on which they have similar outcomes.

\subsection{Scaleup test}
\label{sec_scaleup_test}
The result of a scaleup test using the MNIST8M dataset, which has a total of 8.1 million data points with 784 dimensions, is given in Figure \ref{fig:scaleup}(a). As expected, \texttt{psKC} has runtime linear to data set size, as the data size increases from 10k to 8.1 million (a factor of 810), its runtime increases by a factor of 1,171 (including both feature mapping (GPU) and clustering (CPU)). In contrast, 
as the data size increases from 10k to 40k (a factor of 4),
DP's runtime increases by a factor of 29; and kernel k-means using kNN kernel \cite{KNNKernel} by a factor of 17. Note that beyond 40k points,
DP took too long to run and kernel k-means using kNN kernel had memory error; and the dotted line of DP in Figure \ref{fig:scaleup}(a) is a projected line beyond 40k points. %
The k-means-based SGL is the only contender which runs linearly, but it is still slower than \texttt{psKC}. 

\begin{figure}[t]
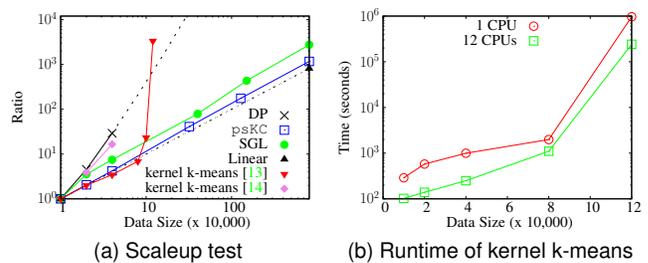

 \centering
 \subfloat[Scaleup test]{\includecombinedgraphics[scale=0.335,vecfile=pset_timing-eps-converted-to]{pset_timing}} \subfloat[Runtime of kernel k-means]{\includecombinedgraphics[scale=0.335,vecfile=kernel_k_means_timing-eps-converted-to]{kernel_k_means_timing}\label{fig:k-mean-runtime}} 
 \caption{Scaleup test on MNIST8M. 
 The base in computing the ratio is the runtime at 10k points. 
 DP crashed 
 at 320K points. `Linear' indicates the runtime of a linear-time algorithm in (a).}
 \label{fig:scaleup}
\end{figure}

Kernel k-means \cite{Scalable-kMeans-JMLR19} has two main components\footnote{Dimensionality reduction using PCA is often an intermediary between feature transformation and k-means in implementations. PCA is used in order to produce the desired rank-restricted Nystrom approximation, as required in their formulation (see Algorithm 2 in \cite{Scalable-kMeans-JMLR19}). However, some exposition may have omitted PCA, though it is used in the code, e.g., \cite{KNNKernel}.}: Nystr\"{o}m approximation \cite{Nystrom_NIPS2000} (to produce a finite dimensional feature map from a kernel of intractable dimensionality) and k-means. 
 
Using a supercomputer Cray XC40 system with 1632 compute
nodes, each has two 2.3GHz 16-core Haswell processors and 128GB of DRAM, the authors \cite{Scalable-kMeans-JMLR19} reported an experiment using scalable kernel k-means with $s=20$ on the MNIST8M dataset as follows: %

The parallelization reduces the runtime of Nystr\"{o}m, but it increases the runtimes of PCA and k-means. Thus, the net speedup is significantly less, e.g., increasing the number of compute nodes 16 times from 8 to 128, the net speedup is less than 4 times. It is even less when the number of target dimensions $s$ is increased. See Table 4 in \cite{Scalable-kMeans-JMLR19} for details. This shows that parallelization alone has diminishing payoff as the complexity of the problem increases. In addition, parallelization needs to increase massively as the data size increases in order to complete in reasonable time. 

Our experimental result in Figure \ref{fig:scaleup}(a) shows the same behaviour on a single-CPU machine.
Scalable kernel k-means has similar runtime ratios as \texttt{psKC} up to 80k. But its runtime began to dramatically increase 
at 120k---the runtime increase is now more than 3000 times on a 12-fold increase in data size!

Even with 12 CPUs, as shown in Figure \ref{fig:scaleup}(b), scalable kernel k-means took more than 240,000 seconds to complete the dataset of 120k points, while the 1-CPU machine took close to 1 million seconds (more than 11 days). This is a speedup of 4 times on a 12-fold increase in the number of CPUs. In other words, the parallelization works well in scalable kernel k-means only if the number of CPUs is sufficiently large such that each CPU works  on a small data set. Otherwise, a dramatic increase in runtime is expected, as shown in Figure \ref{fig:scaleup}(a).

In contrast, the algorithmic advantage of \texttt{psKC}, together with the use of the proposed point-set kernel, allows it to run on a standard machine of single-CPU (for clustering) and GPU (for feature mapping in preprocessing). 
This enables  the clustering to be run on a commonly available machine (with both GPU and CPU) to deal with large scale datasets.

In terms of real time: on the dataset with 40k data points, \texttt{psKC} took 73 seconds which consists of 58	GPU seconds for feature mapping and 15 CPU seconds for clustering. In contrast, DP took 541 seconds. The gap in runtime widens as data size increases: To complete the run on 8.1 million points, DP is projected to take 379 years! That would be 12 billion seconds which is six orders of magnitude slower than \texttt{psKC}'s 20 thousand seconds (less than 6 hours). The widening gap is apparent in Figure \ref{fig:scaleup}(a). SGL took 269 seconds on 40k points, and more than 200 thousand seconds (58 hours) on 8.1 million points.

For the dataset of 120k data points, \texttt{psKC} took 243  seconds;
whereas scalable kernel k-means took close to a million seconds, both on a one-CPU machine.

As it is, there is no opportunity for DP to do feature mapping (where  GPU could be utilised). While it is possible for kernel k-means to make use of GPU as in \texttt{psKC}, scalable kernel k-means's main restriction is PCA which has no efficient parallel implementation, to the best of our knowledge. The clustering procedures of both DP and \texttt{psKC} could potentially be parallelized; but this does not change their time complexities.

\subsection{Stability analysis}
\label{sec_stability}

k-means, used in kernel k-means and SGL, is a randomised algorithm; and its clustering outcome is influenced by the initial random centers. While the \texttt{psKC} procedure is deterministic, the Isolation Kernel employed is generated based on random samples. Here we examine the stability of these algorithms.

Figure \ref{fig:boxplot} shows the stability of the clustering outcomes in terms of F1 score \cite{larsen1999fast,aliguliyev2009performance} over 10 trials, presented in box plots.

\begin{figure}[t]
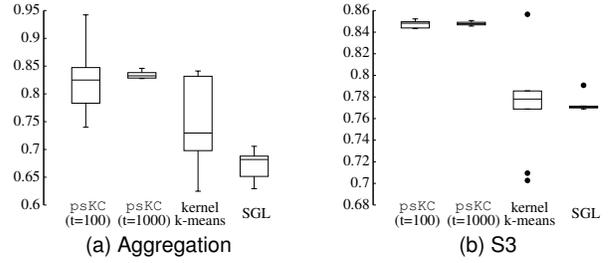

\vspace{-3mm}
 \centering
  \subfloat[Aggregation]{\includecombinedgraphics[scale=0.335,vecfile=bp_aggregation-eps-converted-to]{bp_aggregation}  \label{fig:boxplota}}
 \subfloat[S3]{\includecombinedgraphics[scale=0.335,vecfile=bp_s3-eps-converted-to]{bp_s3}  \label{fig:boxplotb}}

 \caption{Stability of clustering, presented in box plot based on 10 trials of the same parameter setting for each algorithm. The y-axis is F1 score.
 \label{fig:boxplot}
}
\end{figure}

The results show that kernel k-means produced clustering which have variance much higher than those produced by \texttt{psKC} on the middle 50\% results (showed as the box---small (large) box has low (high) variance). Kernel k-means produced wild outliers (see the three points outside the box) as shown on S3. Despite having its best result (the top outlier) is better than all other results,
 its two worst results (the bottom two outliers) are significantly worse than all other results.
Though SGL has lower variance than kernel k-means (because SGL chooses the best outcome from a number of k-means trials internally), its median F1 results (shown as the line inside the box) are worse than those of kernel k-means.

Overall, \texttt{psKC} ($t=100)$ produces higher F1 scores than kernel k-means and SGL.
In addition, Figure \ref{fig:boxplot} also shows that the variance can be significantly reduced by using a higher $t$ at the cost of longer runtime. For example, on S3, \texttt{psKC} ($t=100)$ took 1.0 seconds, and \texttt{psKC} ($t=1000)$ took 9.2 seconds.

\subsection{The effect of post-processing}

Table \ref{tbl_post-processing} shows the effect of post-processing of \texttt{psKC} on two datasets. As mentioned in Section \ref{sec_conceptual-diff}, the clustering outcome of \texttt{psKC} has already achieved a good approximation to the optimal maximization objective. The post-processing provides a final tweak to refine the approximation for the points having the lowest similarity only. On the small dataset S3, the post-processing did not make any re-assignment. On the large dataset Autumn Colors, the post-processing improved the objective function $\Gamma(D)$ by re-assigning three thousand points out of a million points, and it spent 18\% of the total runtime of 10 seconds.

\section{Discussion}
\label{sec_discussion}

\setlength{\tabcolsep}{4pt}
\begin{table}[t]
    \centering
    \caption{The effect of post-processing. Objective function $\Gamma(D)$ is Eq.4.}
    \label{tbl_post-processing} 
    \begin{tabular}{c|r|rrr}
    \hline
    &   {psKC} & \multicolumn{3}{c}{post-processing}\\
    &  $\Gamma(D)$ & $\Gamma(D)$ & \#points re-assigned & time (\%)\\
    \hline
S3 & 1,492 & 1,492 & 0 & 2\% \\ 
Autumn Colors & 294,906 & 297,562 & 3,074 & 18\%  \\
\hline
    \end{tabular}
\end{table}

\subsection{Isolation Kernel versus Gaussian Kernel}

To get the clustering outcomes of \texttt{psKC} we showed here, it is crucial that the point-set kernel employs Isolation Kernel  \cite{ting2018IsolationKernel,IsolationKernel-AAAI2019} which is data dependent. Employing a Gaussian Kernel, which is data independent, \texttt{psKC} will perform poorly on datasets with clusters of non-globular shape, different data sizes and/or densities. This is because its similarity measurement is independent of data distribution. See the clustering outcomes of \texttt{psKC}$_g$ which employs Gaussian kernel in Table \ref{tab:compare_2nd} in Appendix \ref{App_Addtional_result}.

In addition to poor clustering outcomes, the use of Gaussian kernel in point-set kernel has high computational cost. This is because it has a feature map having intractable dimensionality.

\subsection{Issue with running DP on a subsample}
\label{sec_DP_issue}
Our results show that density-based algorithm such as DP is a stronger clustering algorithm than kernel k-means, in terms of clustering outcomes. But DP is one of the most computationally expensive algorithms: it needs large memory space and its runtime is proportional to the square of data size ($n^2$).

It is possible to run DP on an image of high resolution by reducing its resolution first. However, the effect of this reduction can be counterproductive. For example, we have attempted to reduce the resolution of the Chinese painting: Autumn Colors, shown in Table \ref{fig:AutumnColors}. The resolution must be reduced from 1 million pixels to 90,000 pixels for DP to run on a machine with 256GB memory. As a result of this reduction, salient features of the painting were destroyed, e.g., some brush writings have become less obvious. With reduced density of pixels of brush writing, it is no more exhibited as a cluster of interest in its own right in the CIELAB space. Thus, DP is unable to identify the brush writing (together with the painting using the same brush color) as a cluster. An example clustering outcome of DP is shown in Figure~\ref{fig:DP1}. Note that the background and brush writing/painting are not clearly separated into two different clusters by DP, unlike \texttt{psKC}'s clustering outcomes shown in Table \ref{fig:AutumnColors}. In a nutshell, reducing the resolution would add another issue to DP's existing weakness in identifying clusters of uniform distribution as exemplified by the Ring-G dataset shown in Table \ref{tab:compare}.

\begin{figure}[t]
\vspace{-3mm}
\includegraphics[scale=0.5, trim=40 30 0 0, clip]{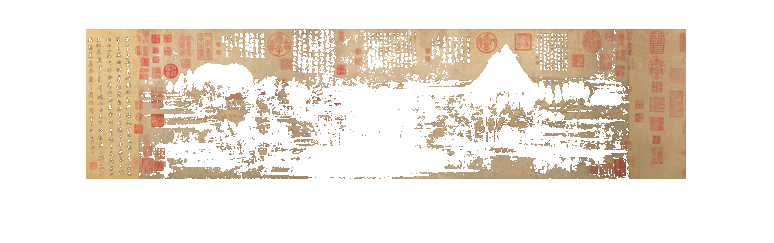}\\
\includegraphics[scale=0.5, trim=40 40 0 30, clip]{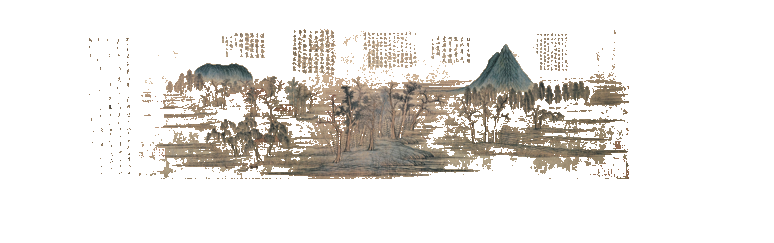}

 \caption{Clustering outcome of DP on the reduced resolution of the Chinese painting: Autumn Colors.}

 \label{fig:DP1}
\end{figure}

\subsection{Relation to Support Vector Clustering}
In the framework of support vector machines, Support Vector Clustering (SVC) \cite{SVC:2001} describes a way to build a minimal enclosing sphere in feature space (induced by a kernel function) to capture the majority of the points in the given dataset. This approach is computationally expensive. Approximation methods have been developed to reduce its computational demands. For example, a subsampling method is proposed to reduce the learning cost in constructing the minimal enclosing sphere, and then employ Voronoi cells to assign points to different clusters \cite{Voronoi-cell-based-clustering}. This subsampling approximation method suffers from the same issue we described in the last subsection.

The only connection between  \texttt{psKC} and SVC is the use of a kernel. Still, \texttt{psKC} employs a  data dependent kernel called Isolation Kernel. Like all existing kernel-based methods, SVC employs a data independent kernel such as Gaussian kernel. 

\subsection{Some algorithmic weakness cannot be overcome by changing the distance function}

A previous work \cite{IsolationKernel-AAAI2019} has shown that the clustering outcome of DBSCAN \cite{DBSCAN-1996} can be improved
by simply replacing Euclidean distance with Isolation Kernel. However, the resultant algorithm called MBSCAN \cite{IsolationKernel-AAAI2019}, though performs better than DBSCAN on the Ring-G and S3 datasets shown in Table \ref{tab:compare_2nd} in Appendix \ref{App_Addtional_result}, MBSCAN still performs poorer than \texttt{psKC} on the S3 dataset. 

This is another example of the limitation of existing clustering procedures: DBSCAN has some inherent algorithmic weakness that cannot be overcome by using a data dependent kernel.

\subsection{Current approaches to runtime issue}
A considerable amount of research has been delegated to mitigate the persistent longstanding runtime issue of using a point-to-point distance/kernel. For example, various indexing techniques \cite{LSH-2015,DBSCAN-Revisited:2017,TingLiu2007} have been explored to reduce the $n^2$ runtime. While some have claimed to have achieved runtime proportional to $n$, this often comes with the cost of reduced task-specific performance \cite{LSH-2015,DBSCAN-Revisited:2017, TingLiu2007}.  

For the point-to-point kernel-based methods, an alternative to indexing is to use kernel functional approximation \cite{Nystrom_NIPS2000,Nystrom-NIPS12} to speed up their runtimes, e.g., the scalable kernel k-means \cite{Scalable-kMeans-JMLR19} mentioned above. 

Both approaches of indexing and kernel functional approximation reduce runtime by sacrificing task-specific performance.
In contrast, \texttt{psKC} has its runtime proportional to $n$, \textbf{needing neither indexing nor kernel functional approximation}. As a result, its clustering outcomes are not compromised.

Parallelization is a way to distribute the workload, with/without any of the two approaches mentioned above. A massive parallelization enables an algorithm to run on a large dataset that was impossible with a single CPU. Given a fixed number of CPUs, it sets a limit on the data size that it can handle,  as shown in Section \ref{sec_scaleup_test} with scalable kernel k-means. A massive parallelization could also be used to scale up \texttt{psKC}, like any other algorithms. In other words, what we have presented in this paper is more fundamental than and orthogonal to parallelization. 

\subsection{Linear-time clustering algorithms}

k-means \cite{k-means-macqueen1967} is the clustering algorithm of choice in many applications \cite{AgarwalDMBook2015}, including text clustering, graph clustering and hierarchical clustering (where k-means is the core clustering algorithm). This is because it often has linear runtime in practice \cite{AgarwalDMBook2015}, though its time complexity is superpolynomial in the worst case which requires  $2^{\Omega{(\sqrt{n})}}$ iterations \cite{k-means-SCG2006}. Being a linear-time kernel clustering algorithm, where the maximum number of iterations is fixed for some range of parameter settings,  psKC offers an alternative to k-means in these and many other applications.

\section{Conclusions} 
We show that the proposed clustering \texttt{psKC} outclasses  DP, DBSCAN and two versions of kernel k-means as well as other variants such as SGL and RCC in terms of both clustering outcomes and runtime efficiency. We identify that the  two root causes of shortcomings of existing clustering algorithms are (i) the use of \emph{data independent} point-to-point distance/kernel (where the kernel has a feature map with intractable dimensionality) to compute the required similarity; and (ii) the algorithmic designs that constrict the types of clusters that they can identify. For example, kernel k-means can be interpreted as already using a point-set kernel. Yet, because its use is restricted to represent cluster centers, only clusters of globular shape and approximately the same size can be detected in feature space; and this does not guarantee that clusters of non-globular shape and different sizes in input space can be detected.
This is also the cause of its sensitivity to outliers. 
Both root causes have led to poor clustering outcomes. 

The first root cause is the source of the longstanding runtime issue in density-based clustering algorithms that has prevented them from handling large scale datasets. 

We address these root causes by using a \emph{data dependent point-set kernel} and a new clustering algorithm which utilizes the point-set kernel as a medium to grow and characterize clusters---enabling clusters of arbitrary shapes and sizes to be detected and they are insensitive to outliers. As a result, \texttt{psKC} is the only clustering algorithm which is both effective and efficient---a quality which is all but nonexistent in current clustering algorithms. It is also the only kernel-based clustering that has runtime proportional to data size.

We show that 
the number of iterations in \texttt{psKC} is independent of data size. The parameter search ranges in \texttt{psKC} can be set to bound the number of iterations; this gives the linear time complexity in practice, enabling \texttt{psKC} to deal with large datasets. 

\section*{Acknowledgement}
Kai Ming Ting is supported by Natural Science Foundation
of China (62076120).

\bibliography{references}
\bibliographystyle{IEEEtran}

\clearpage
\newpage

\appendices
\section{Isolation Kernel}
\label{App_IK}
We provide the pertinent details of Isolation Kernel in this section. Other details can be found in \cite{ting2018IsolationKernel,IsolationKernel-AAAI2019}.

Let $D=\{{x}_1,\dots,{x}_n\}, {x}_i \in \mathbb{R}^d$ be a dataset sampled from an unknown probability density function ${x}_i \sim F$.  
Let $\mathds{H}_\psi(D)$ denote the set of
all partitionings $H$
that are admissible under $D$, 
where each $H$ covers the entire space of $\mathbb{R}^d$. Each of the $\psi$ isolating partitions $\theta[{z}] \in H$ isolates one data point ${z}$ from the rest of the points in a random subset $\mathcal D \subset D$, and $|\mathcal D|=\psi$, where $z \in \mathcal D$ has the equal probability of being selected from $D$.

\begin{definition} For ${x}, {y} \in \mathbb{R}^d$,
	Isolation Kernel of ${x}$ and ${y}$ wrt $D$ is defined to be
	the expectation taken over the probability distribution on all partitionings $H \in \mathds{H}_\psi(D)$ that both ${x}$ and ${y}$  fall into the same isolating partition $\theta[{z}] \in H, {z} \in \mathcal{D}$:
\begin{eqnarray}
\kappa_\psi({x},{y}\ |\ D) &=&  {\mathbb E}_{\mathds{H}_\psi(D)} [\mathds{1}({x},{y} \in \theta[{z}]\ | \ \theta[{z}] \in H)] 
\label{eqn_kernel}
	\end{eqnarray}
where $\mathds{1}(\cdot)$ is an indicator function.
\end{definition}

In practice, $\kappa_\psi$ is constructed using a finite number of partitionings $H_i, i=1,\dots,t$, where $H_i$ is created using $\mathcal{D}_i \subset D$:
\begin{eqnarray}
\kappa_\psi({x},{y}|D)  & = &  \frac{1}{t} \sum_{i=1}^t   \mathds{1}({x},{y} \in \theta\ | \ \theta \in H_i) \nonumber\\
 & = & \frac{1}{t} \sum_{i=1}^t \sum_{\theta \in H_i}   \mathds{1}({x}\in \theta)\mathds{1}({y}\in \theta) 
 \label{Eqn_IK}
\end{eqnarray}

\noindent
where $\theta$ is a shorthand for $\theta[{z}]$; and $\kappa_\psi({x},{y}|D) \in [0,1]$.

Isolation Kernel is positive semi-definite as Eq \ref{Eqn_IK} is a quadratic form, i.e., it defines a Reproducing Kernel Hilbert Space.

\textbf{aNNE Implementation}:
As an alternative to using trees \cite{liu2008isolation} in its first implementation of Isolation Kernel \cite{ting2018IsolationKernel}, a nearest neighbour ensemble (aNNE) has been used instead \cite{IsolationKernel-AAAI2019}. 

Like the tree method, the nearest neighbour method also produces each $H$ model which consists of $\psi$ isolating partitions $\theta$, given a subsample of $\psi$ points. Rather than representing each isolating partition as a hyper-rectangle, it is represented as a cell in a Voronoi diagram, where the boundary between two points is the equal distance from these two points.

$H$, being a Voronoi diagram, is built by employing $\psi$ points in $\mathcal D$,
where each isolating partition or Voronoi cell $\theta \in H$ isolates one data point from the rest of the points in $\mathcal D$. The point which determines a cell is regarded as the cell centre.

Given a Voronoi diagram $H$ constructed from a sample $\mathcal{D}$ of $\psi$ points, the Voronoi cell centred at $z \in \mathcal{D}$ is:
\[ \theta[z] = \{x  \in \mathbb{R}^d \ | \  z = \argmin_{\mathsf{z} \in \mathcal{D}} \parallel x - \mathsf{z} \parallel_2 \}. \]

Note that the boundaries of a Voronoi diagram are derived implicitly to be equal distance between any two points in $\mathcal{D}$; and they need not be derived explicitly in realising Isolation Kernel.

\begin{figure}[!t]
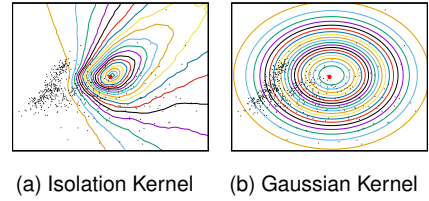

\vspace{-3mm}
 \centering
 \subfloat[Isolation Kernel]{\includecombinedgraphics[height=.12\textwidth,width=.16\textwidth,vecfile=aNNE_forest_psi_10_t_1000_lp_2-eps-converted-to]{aNNE_forest_psi_10_t_1000_lp_2}}
 \subfloat[Gaussian Kernel]{\includecombinedgraphics[height=.12\textwidth,width=.16\textwidth,vecfile=gaussian_forest_sigma_0_2_lp_2-eps-converted-to]{gaussian_forest_sigma_0_2_lp_2}}
 \caption{Contour plots of Isolation Kernel and Gaussian Kernel on a real-world dataset. 
 Black dots are data points. }
\vspace{-2mm}
\label{fig_contour_example}
\vspace{-3mm}
\end{figure}

\subsection{Kernel distributions in contour plots}

Figure \ref{fig_contour_example} compares the contour plots of of Isolation Kernel and Gaussian Kernel. Notice that each contour line of Isolation Kernel, which denotes the same similarity to the centre (red point), is elongated along the sparse region and compressed along the dense region. In contrast, Gaussian kernel (or any data independent kernel) has the same symmetrical contour lines around the centre point, independent of data distribution (as shown in Figure \ref{fig_contour_example}(b)).

\subsection{Feature map of Isolation Kernel}
\label{sec_feature_map}

\begin{definition}
	\label{def:featureMap}
	\textbf{Feature map of Isolation Kernel.}
	For point $x \in \mathbb{R}^d$, the feature mapping $\Phi: x\rightarrow \mathbb \{0,1\}^{t\times \psi}$ of $\kappa_\psi$ is a vector that represents the partitions in all the partitioning $H_i\in \mathds{H}_\psi(D)$, $i=1,\dots,t$; where $x$ falls into only one of the $\psi$ Voronoi cells in each partitioning $H_i$.
\end{definition}

Given $H_i$, $\Phi_i(x)$ is a $\psi$-dimensional binary column vector representing all Voronoi cells $\theta_j \in H_i$, $j=1,\dots,\psi$; where $x$ falls into only one of the $\psi$ Voronoi cells.
The $j$-component of the vector is:
$\Phi_{ij}(x)=\mathds{1}(x\in \theta_j\ |\ \theta_j\in H_i)$. Given $t$ partitionings, $\Phi(x)$ is the concatenation of $\Phi_1(x),\dots,\Phi_t(x)$.

A commonly used kernel such as Gaussian kernel is defined with a function; and its feature map is derived from this function. As a result, the kernel can be computed without a feature map.

Instead, 
Isolation Kernel, having no functional form, is defined based on a finite-dimensional feature map, defined from samples of a given dataset. The samples, as well as the resultant partitions, are products of randomized processes. Once they are determined in the preprocessing step, the feature map defines the Isolation Kernel \emph{exactly} in its subsequent use in an algorithm.

\section{Point-set kernel: Data dependent property}
\label{App_Property_psKC}
The data dependent property of point-set kernel follows directly from the data dependent property of Isolation Kernel. Its  Lemma \cite{IsolationKernel-AAAI2019} is re-stated as follows:
\begin{lem}
\label{lem_IK}
$\forall x, y \in \mathcal{X}_\mathsf{T}$ (dense region) and $\forall x',y' \in \mathcal{X}_\mathsf{S}$ (sparse region) such that $\forall_{z\in \mathcal{X}_\mathsf{S}, z'\in \mathcal{X}_\mathsf{T}} \ \rho(z) > \rho(z')$,
the nearest neighbour-induced Isolation Kernel $\kappa$ has the characteristic that for $\ell_p(x-y)\ =\ \ell_p(x'- y')$ implies
$\kappa( x, y\ |\ D) < \kappa( x', y'\ |\ D)$.
\end{lem}

\noindent
Let $\Bar{x}_C$ and $\Bar{x}_{C'}$ be the preimages of $\widehat{\Phi}(C|D)$ and $\widehat{\Phi}(C'|D)$, respectively; and $\rho(C) > \rho(C') \equiv \rho(\Bar{x}_C) > \rho(\Bar{x}_{C'})$. 

\noindent
Further let two reference points $x \in C$ and $x' \in C'$ such that $\ell(x,C) = \ell(x',{C'}) \equiv \ell_p(x-\Bar{x}_C)\ =\ \ell_p(x'- \Bar{x}_{C'})$. We then have reference values: $\kappa(x,\Bar{x}_C) < \kappa(x',\Bar{x}_{C'})$; or for tiny distance $\ell(x,C)$: $\kappa(x,\Bar{x}_C) \approx
\kappa(x',\Bar{x}_{C'})$.

\noindent
From Lemma \ref{lem_IK},  
because $\kappa$ decreases at a faster rate in $\mathcal{X}_\mathsf{T}$ than that in $\mathcal{X}_\mathsf{S}$ as $\Delta x$ increases (i.e., the probability of two points of equal inter-point distance falling into the same Voronoi cell is a monotonically decreasing function wrt the density of the cell, see the proof of the Lemma in \cite{IsolationKernel-AAAI2019}), the following holds:

$\kappa(x+\Delta x, \Bar{x}_C) < \kappa(x'+\Delta x, \Bar{x}_{C'})$
\&\\
$[\kappa(x,\Bar{x}_C) - \kappa(x+\Delta x,\Bar{x}_C)] > [\kappa(x',\Bar{x}_{C'}) - \kappa(x'+\Delta x,\Bar{x}_{C'})]$.

\noindent
Because $\kappa(x,\Bar{x}_C) \equiv \widehat{K}(x,C)$,
\[
\frac{\widehat{K}(x,C) - \widehat{K}(x+\Delta x,C)}{\Delta x} > \frac{\widehat{K}(x',C') - \widehat{K}(x'+\Delta x,C')}{\Delta x}.
\]

\noindent
This gives
$\frac{d \widehat{K}(x,C)}{dx} > \frac{d \widehat{K}(x',C')}{dx}$, if $\ell(x,C) = \ell(x',C')$  and $\rho(C) > \rho(C')$.

\section{kernel k-means versus spectral clustering}
\label{kernel-k-means-other-issues}

The close relationship between kernel k-means and spectral clustering is well established \cite{kernel-k-means-2004, Spectral_clustering-NIPS2003}. 
Kernel k-means and many spectral clustering methods use the same k-means algorithm.
In order to make them scalable for large datasets, kernel k-means usually employs a kernel functional approximation (e.g., Nystr\"{o}m \cite{Nystrom_NIPS2000}) as a preprocessing. An example is scalable kernel k-means \cite{Scalable-kMeans-JMLR19}, and it has a direct comparison with a spectral clustering that employs the same Nystr\"{o}m, utilizing a small set of `landmark' points. 
SGL \cite{StructureGraphLearning-2021} is another example that employs a small set of landmark points to deal with large datasets, though the problem formulation differs. It has close relation to spectral clustering because of the use of the idea of graph embedding.

Spectral clustering can be viewed more broadly as embedding a graph into Euclidean space via an eigenvalue decomposition, before a clustering algorithm is applied. The embedding could be done in several ways; so as the clustering (though typically k-means, a clustering based on Gaussian Mixture model has also been explored, e.g., \cite{SpectralGraphClustering-PNAS5995}.)

Because of the eigenvalue decomposition, spectral clustering is usually more computationally expensive than kernel k-means. Converting a spectral clustering problem into a weighted kernel k-means problem is one way to reduce its time complexity \cite{Kernel-kmeans-spectral-clustering}. A recent connection between spectral clustering and kernel k-means is made via regularisation \cite{KernelCut-2019}.

\section{Experimental settings \& additional results}
\label{App_Addtional_result}
This section provides the details of the datasets/images used, the experimental settings and additional results.  

\vspace{3mm}
\noindent
\textbf{Datasets used}

\begin{enumerate}
    \item \href{http://www.cs.uef.fi/sipu/datasets/}{\underline{Artificial benchmark datasets}}: all benchmark datasets are from \cite{ClusteringDatasets}, except the Ring-G dataset which is our creation, and the AC dataset was first used in \cite{KNNKernel}.
      \item \href{https://www.paintinghere.com/painting/vincent_van_gogh_starry_night_over_the_rhone_2_4713.html}{ \underline{Starry Night over the Rhone 2}} (932 x 687)
    \item \href{www.comuseum.com}{\underline{Zhao Mengfu’s Autumn Colors}} (2,005 x 500)
    \item \href{https://www.csie.ntu.edu.tw/~cjlin/libsvmtools/datasets/multiclass.html}{\underline{MNIST data sets}}: images of handwritten digits, each is represented with 784 dimensions: 
    
    MNIST8M has 8.1 million images \cite{loosli-canu-bottou-2006}. 
    
    MNIST70k is the source dataset (having 70,000 images) from which the MNIST8M dataset is generated.

\end{enumerate}

\noindent
\textbf{Experimental settings}.
We search parameters in each of the algorithms, i.e., DP, scalable kernel k-means, kernel k-means and \texttt{psKC}; and report their best clustering outcomes after the search. 

We implemented \texttt{psKC} in C++. Scalable kernel k-means is implemented in Scala as part of the Spark framework \cite{Scalable-kMeans-JMLR19}; DBSCAN is implemented in Java as part of the WEKA framework \cite{Weka-reference}; RCC is implemented in Python \cite{shah2017robust}; and DP, DP$_{ik}$, k-means and kNN kernel are implemented in Matlab \cite{KNNKernel,DensityPeak-2014}.

The parameter search ranges used in the experiments are:

\begin{itemize}[leftmargin=4mm]
    \item DP:
$\epsilon$ (the bandwidth used for density estimation) is in [0.001m, 0.002m,..., 0.4m] where $m$ is the maximum pairwise distance.
\item kernel k-means \cite{KNNKernel}: $k$ in kNN kernel is in [0.01n, 0.02n,..., 0.99n]; and the number of dimensions used is 100. 
\item Scalable kernel k-means \cite{Scalable-kMeans-JMLR19}: $\sigma$ in [0.1, 0.25, 0.5, 1,..., 16, 24, 32]; 
$s=100$ (target \#dimensions of the PCA step) and $c=400$ (\#dimensions of Nystr\"{o}m's output), except for data set less than 400 points then it is $s=20$ and $c=200$. 
\item \texttt{psKC}: $\psi$ in [2, 4, 6, 8, 16, 24, 32], $\tau=0.1$ and $\varrho=0.1$ for the images, except MNIST70k: $\psi=5000$, $\tau=0.001$.
For artificial datasets, $\psi$ in [55, 70, 128, 256, 512, 1024], $\tau$ in [0.1, 0.2, 0.5, 1, 2, 5, 10, 20, 50, 100, 800]$\times10^{-4}$, $\varrho=0.1$ (except in Ring-G, where $\varrho=0.26$ was used.)

\item \texttt{psKC$_{g}$}: $\gamma = 2^i$ where $i$ in [1, 2, 3,..., 16], $\tau=0.1$ and $\varrho$ in [0.1, 0.01, 0.001,.., $1 \times 10^{-10}$].
\item DBSCAN: $\epsilon$ in [0.001m, 0.002m,..., 0.999m] and $MinPts$ in [2, 3,..., 30], where $m$ is the maximum pairwise distance and $MinPts$ is the density threshold.
\item MBSCAN: $\psi=32$, $t=200$, $\epsilon$ in [0.001m, 0.002m,..., 0.999m] and $MinPts$ in [2, 3,..., 30], where $m$ is the maximum pairwise distance and $MinPts$ is the density threshold. 
\item DP$_{ik}$: For DP, $\epsilon$ is in [0.001m, 0.002m,..., 0.4m] where $m$ is the maximum pairwise distance. 
Setting for Isolation Kernel: $\psi$ in [2, 4, 6, 8, 16, 24, 32].
\item RCC: $k$ is in [1, 2,..., 200] and $\tau$ is in [1, 1.1,..., 20].
\item SGL \cite{StructureGraphLearning-2021}: $\#anchor$ in [25, 30, 33, 36, 40, 50], $\alpha$ in [0.1, 1, 10, 50] and $\beta$ in [0.0001, 0.001]
\end{itemize}

All algorithms were given the advantage by setting the number of clusters to the true cluster number, except DBSCAN, MBSCAN, RCC and \texttt{psKC} which find the number of clusters automatically through their parameter settings.
$t=100$ is the default setting for Isolation Kernel for all experiments, except those in Section~\ref{sec_stability}.

The experiments were run on a Linux CPU machine: AMD 16-core CPU with each core running at 2.0 GHz and 32 GB RAM. The feature space conversion was executed on a machine having GPU: 2 x GTX 1080 Ti with each card having 12 GB RAM.

\vspace{2mm}
\noindent
\textbf{Guide for \texttt{psKC} parameter tuning}: In addition to $\varrho=0.1$ which can usually be set as default, the following two rules may be used to tune \texttt{psKC}:

1. If clusters are joined but need to be split, then increase $\psi$.

2. If clusters are split but need to be joined, then decrease $\tau$.

\vspace{2mm}
\noindent
\textbf{Parameter settings used by }\texttt{psKC} in to produce the clustering outcomes shown in Tables \ref{tab:compare} to \ref{fig_mnist70k} are listed in  Table \ref{tab:psKC_param}.

\vspace{2mm}
\noindent
\textbf{Source code} of \texttt{psKC} is available at\\ \url{https://github.com/IsolationKernel/Codes/tree/main/IDK}

\vspace{2mm}
\noindent
\textbf{Additional results}:
Table \ref{tab:compare_2nd} shows the clustering outcomes of other clustering algorithms, in addition to those used in the main text. The additional algorithms are DBSCAN \cite{DBSCAN-1996}, MBSCAN \cite{IsolationKernel-AAAI2019}, RCC\footnote{Robust continuous clustering \cite{shah2017robust} optimises a continuous and differentiable objective based on the duality between robust estimation and line processes \cite{black1996unification}.
Its objective is applied on the connectivity of a mutual $k$-nearest neighbours  (m-$k$NN) graph \cite{brito1997connectivity} and is optimised using a linear least-squares solver. RCC has the ability to detect arbitrarily shaped clusters based on the  m-$k$NN graph, but its computational complexity is dominated by the m-$k$NN graph construction which costs $n^2$
\cite{dong2011efficient}. Thus, although the solver used in RCC scales to large datasets, RCC is still unable to deal with large datasets.} \cite{shah2017robust}, k-means \cite{k-means-macqueen1967}, SGL \cite{StructureGraphLearning-2021}, DP$_{ik}$ and \texttt{psKC}$_g$. Both MBSCAN and DP$_{ik}$ employ Isolation Kernel.
\texttt{psKC}$_g$ employs the following point-set kernel, where $\kappa_g$ is the Gaussian kernel:
\[
\widehat{K}_g(x,G) = \frac{1}{|G|}\sum_{y \in G} \kappa_g(x,y)
\]

The result in Table \ref{tab:compare_2nd} shows that MBSCAN which employs IK has the best clustering outcomes. It is the only one among these algorithms that correctly cluster four out of the five datasets. But, MBSCAN is still worse than \texttt{psKC} in terms of clustering outcomes and time complexity.

\begin{table*}
 \caption{The \texttt{psKC} parameters used to generate the clustering outcomes in Tables 2, 3, 4 \& 5. The random seed was set to 42, $t=100$ and $\sigma = 0.1$ for all experiments, except $\sigma = 0.26$ for Ring-G.}
 \label{tab:psKC_param}
 \centering
 \begin{tabular}{clrrc|clrr}
  \toprule
Table & Dataset & \multicolumn{1}{c}{$\psi$} & \multicolumn{1}{c}{$\tau$} & & Table & Dataset & \multicolumn{1}{c}{$\psi$} & \multicolumn{1}{c}{$\tau$} \\
  \midrule
\multirow{5}{*}{2} & Ring-G & 128 & $2 \times 10^{-3}$ & & \multirow{2}{*}{3} & Autumn Colors & 8 & \multirow{2}{*}{0.1} \\
 & AC & 256 & $2 \times 10^{-4}$ & &  & Red stamps (Class 7) & 24 \\ \cmidrule{6-9}
 & Aggreg & 128 & $1 \times 10^{-2}$ & & 4 & Starry Night & 16 & 0.1\\ \cmidrule{6-9}
 & Spiral & 512 & $1 \times 10^{-4}$ & & 5 & MNIST70k & 5,000 & $1 \times 10^{-3}$ \\
 & S3 & 70 & $8 \times 10^{-2}$ & \\
  \bottomrule
 \end{tabular}
\end{table*}

\setlength{\fboxrule}{2pt}
\begin{table*}[h]
 \caption{Artificial datasets: Clustering outcomes of k-means, DBSCAN, \texttt{psKC$_{g}$} which employs Gaussian Kernel, MBSCAN \& DP$_{ik}$ which employ Isolation Kernel and RCC. 
 The results with yellow frames indicate good clustering outcomes; and those without have poor clustering outcomes.} \label{tab:compare_2nd}
 \centering
 \setlength{\tabcolsep}{0pt}
\tcbset{width=2.8cm,boxsep=0mm,left=-2mm,right=0mm,top=0mm,bottom=0mm,colframe=yellow!75!white,colback=black!0!white}
 \begin{tabular}{@{}c@{}c@{}c@{}c@{}c@{}c@{}c@{}}
  \toprule
 & k-means & \texttt{psKC$_{g}$} & DBSCAN & MBSCAN & DP$_{ik}$ & RCC \\ 
  \midrule

\begin{turn}{90}  \hspace{1.cm}Ring-G   
\end{turn} &  \includecombinedgraphics[scale=0.38,vecfile=2U_ring_2G_kmeans-eps-converted-to]{2U_ring_2G_kmeans} & \includecombinedgraphics[scale=0.38,vecfile=2U_ring_2G_gaus-eps-converted-to]{2U_ring_2G_gaus} & \includecombinedgraphics[scale=0.38,vecfile=2U_ring_2G_dbscan-eps-converted-to]{2U_ring_2G_dbscan} & \begin{tcolorbox}\includecombinedgraphics[scale=0.38,vecfile=2U_ring_2G_MBSCAN-eps-converted-to]{2U_ring_2G_MBSCAN} \end{tcolorbox} & \includecombinedgraphics[scale=0.38,vecfile=2U_ring_2G_DPik-eps-converted-to]{2U_ring_2G_DPik} & 
\includecombinedgraphics[scale=0.38,vecfile=rcc_ring_g-eps-converted-to]{rcc_ring_g} \\
\begin{turn}{90}  \hspace{1.5cm}AC \end{turn}
& \includecombinedgraphics[scale=0.38,vecfile=k_means_AC_mkVI-eps-converted-to]{k_means_AC_mkVI} & \begin{tcolorbox}\includecombinedgraphics[scale=0.38,vecfile=psKC_gaus_2C_new_mkVI-eps-converted-to]{psKC_gaus_2C_new_mkVI} \end{tcolorbox} & \begin{tcolorbox}\includecombinedgraphics[scale=0.38,vecfile=dbscan_AC_mkVI-eps-converted-to]{dbscan_AC_mkVI}\end{tcolorbox} &
\begin{tcolorbox}\includecombinedgraphics[scale=0.38,vecfile=dbscan_AC_mkVI-eps-converted-to]{dbscan_AC_mkVI} \end{tcolorbox} & \begin{tcolorbox}\includecombinedgraphics[scale=0.38,vecfile=dp_ik_AC_mkVI-eps-converted-to]{dp_ik_AC_mkVI} \end{tcolorbox} & \begin{tcolorbox}\includecombinedgraphics[scale=0.38,vecfile=dbscan_AC_mkVI-eps-converted-to]{dbscan_AC_mkVI} \end{tcolorbox}\\
\begin{turn}{90}  \hspace{1.cm}Aggreg. \end{turn}
& \includecombinedgraphics[scale=0.38,vecfile=k_means_aggregation-eps-converted-to]{k_means_aggregation}  & \includecombinedgraphics[scale=0.38,vecfile=psKC_gaus_aggregation-eps-converted-to]{psKC_gaus_aggregation} & \begin{tcolorbox}\includecombinedgraphics[scale=0.38,vecfile=dbscan_aggregation-eps-converted-to]{dbscan_aggregation} \end{tcolorbox} &
\begin{tcolorbox}\includecombinedgraphics[scale=0.38,vecfile=dbscan_aggregation-eps-converted-to]{dbscan_aggregation} \end{tcolorbox} &
\begin{tcolorbox}\includecombinedgraphics[scale=0.38,vecfile=dp_ik_aggregation-eps-converted-to]{dp_ik_aggregation} \end{tcolorbox} & \begin{tcolorbox}\includecombinedgraphics[scale=0.38,vecfile=dbscan_aggregation-eps-converted-to]{dbscan_aggregation} \end{tcolorbox} \\
\begin{turn}{90}  \hspace{1.cm}Spiral \end{turn}
& \includecombinedgraphics[scale=0.38,vecfile=k_means_spiral-eps-converted-to]{k_means_spiral}  & \begin{tcolorbox}\includecombinedgraphics[scale=0.38,vecfile=psKC_gaus_spiral-eps-converted-to]{psKC_gaus_spiral} \end{tcolorbox}  & \begin{tcolorbox}\includecombinedgraphics[scale=0.38,vecfile=dbscan_spiral-eps-converted-to]{dbscan_spiral} \end{tcolorbox} &
\begin{tcolorbox}\includecombinedgraphics[scale=0.38,vecfile=dbscan_spiral-eps-converted-to]{dbscan_spiral} \end{tcolorbox} &
\begin{tcolorbox}\includecombinedgraphics[scale=0.38,vecfile=dp_ik_spiral-eps-converted-to]{dp_ik_spiral} \end{tcolorbox} & \begin{tcolorbox}\includecombinedgraphics[scale=0.38,vecfile=dbscan_spiral-eps-converted-to]{dbscan_spiral} \end{tcolorbox}\\
\begin{turn}{90}  \hspace{1.cm}S3 \end{turn}
& \includecombinedgraphics[scale=0.38,vecfile=k_means_s3-eps-converted-to]{k_means_s3} & \includecombinedgraphics[scale=0.38,vecfile=psKC_gaus_s3_15_cls-eps-converted-to]{psKC_gaus_s3_15_cls} & \includecombinedgraphics[scale=0.38,vecfile=dbscan_s3-eps-converted-to]{dbscan_s3} & \includecombinedgraphics[scale=0.38,vecfile=mbscan_s3-eps-converted-to]{MBSCAN_s3} & \includecombinedgraphics[scale=0.38,vecfile=dp_ik_s3-eps-converted-to]{dp_ik_s3} & \includecombinedgraphics[scale=0.38,vecfile=rcc_s3-eps-converted-to]{rcc_s3}\\
\bottomrule
 \end{tabular}

\end{table*}

\end{document}